\documentclass{article}

\PassOptionsToPackage{numbers,sort&compress}{natbib}
\usepackage[preprint]{neurips_2026}

\usepackage{microtype}
\usepackage{subcaption}
\usepackage{booktabs} 
\usepackage{multirow}
\usepackage{colortbl}
\usepackage{xcolor}
\usepackage{multicol}
\usepackage{mdframed}
\usepackage{enumitem}
\usepackage{hyperref}
\newcommand{\best}[1]{\cellcolor{blue!8}\textbf{#1}}
\usepackage{booktabs}
\usepackage{tabularx}
\usepackage{array}
\usepackage{makecell}
\usepackage[table]{xcolor}

\newcolumntype{Y}{>{\raggedright\arraybackslash}X}

\usepackage{algorithm}
\usepackage{algorithmic}
\usepackage{amsmath}
\usepackage{amssymb}
\usepackage{mathtools}
\usepackage{amsthm}
\usepackage{array}
\usepackage{booktabs}
\usepackage{tabularx}
\usepackage[table]{xcolor}

\usepackage[capitalize,noabbrev]{cleveref}

\theoremstyle{plain}
\newtheorem{theorem}{Theorem}[section]

\theoremstyle{definition}

\theoremstyle{remark}

\usepackage{comment}
 
\usepackage{mdframed}      
\usepackage[textsize=tiny]{todonotes}

\title{HeRo-Q: A General Framework for Stable Low-Bit Quantization via Hessian Conditioning}

\author{
Jinhao Zhang$^{\spadesuit}$,
Yunquan Zhang$^{\clubsuit}$,
Daning Cheng$^{\clubsuit}$,
Boyang Zhang$^{\clubsuit, \dagger}$,
Jun Sun$^{\heartsuit}$,
Zicheng Yan$^{\diamondsuit}$ \\
\\
$^{\spadesuit}$Beijing University of Posts and Telecommunications \\
$^{\clubsuit}$Institute of Computing Technology, Chinese Academy of Sciences \\
$^{\diamondsuit}$University of Science and Technology of China \\
$^{\heartsuit}$Zhejiang Lab \\
$^{\dagger}$Peng Cheng Laboratory \\
}

\begin{document}
\maketitle
\begin{abstract}
Post-Training Quantization (PTQ), a mainstream model compression technique, often leads to the paradoxical “low-error, high-loss” phenomenon because it focuses solely on minimizing quantization error. The root cause lies in the Hessian matrix of the LLM loss landscape: a few high-curvature directions are extremely sensitive to perturbations. To address this, we propose the Hessian Robust Quantization (HeRo-Q) algorithm, which applies a lightweight, learnable rotation-compression matrix to the weight space before quantization. HeRo-Q quantizes weights in a Hessian-conditioned coordinate system, suppressing sharp effective-curvature directions and improving robustness to low-bit noise. HeRo-Q requires no architectural modifications, incurs negligible computational overhead, and integrates seamlessly into existing PTQ pipelines. Experiments on Llama and Qwen models show that HeRo-Q consistently outperforms state-of-the-art methods—including GPTQ, AWQ, and SpinQuant—not only achieving superior performance under standard W4A8 settings, but also excelling in the highly challenging W3A16 ultra-low-bit regime, where it boosts GSM8K accuracy on Llama-3.1-8B to 70.15\% and effectively avoids the logical collapse commonly seen in aggressive quantization.
\end{abstract}

\section{Introduction}

Post-training quantization (PTQ) has become a standard approach for reducing
the serving cost of large language models without requiring
full retraining \cite{AWQ,GPTQ,SmoothQuant,spinquant}. Current mainstream PTQ methods primarily aim to minimize quantization error. They typically employ symmetric or asymmetric linear mapping schemes, use a calibration dataset to estimate dynamic ranges, and select optimal quantization parameters by minimizing statistical discrepancies between pre-quantized and post-quantized weights or activations. 
 
 Although these strategies work well for traditional models like those in computer vision, they often encounter a “low-error, high-loss” paradox when applied to modern LLMs \cite{DBLP:journals/corr/abs-2502-15802,DBLP:conf/nips/ZhangCDL0L24,DBLP:conf/icml/0001GC00XX25}:  even with extremely low overall quantization error, downstream task performance can still degrade substantially, as shown in Figure \ref{fig:low_error_high_loss}.  This suggests the existence of highly sensitive parameters that, despite contributing little to overall noise, exert a disproportionately large impact on model behavior \cite{DBLP:conf/icml/0001GC00XX25}.

A key explanation for this degradation lies in the geometric structure of the LLM loss landscape—particularly the ill-conditioned nature of its Hessian matrix \cite{DBLP:journals/corr/abs-2304-01089,DBLP:conf/nips/HeNPSH24,DBLP:conf/iclr/An000W25}. Specifically, the Hessian spectrum of LLMs is highly non-uniform: a few directions  exhibit very large eigenvalues (i.e., high curvature), while most directions have near-zero eigenvalues (i.e., flat regions). This strong anisotropy makes the model exquisitely sensitive to tiny perturbations along high-curvature directions. Because low-bit rounding errors generally contain components across many coordinate directions, even a small overall error norm can lead to substantial loss increase when a non-negligible component is projected onto high-curvature directions, as shown in Figure \ref{fig:loss_landscape_geometry}.
\begin{mdframed}[leftline=true, rightline=false, topline=false, bottomline=false, linewidth=2.5pt, linecolor=magenta!50, backgroundcolor=gray!10]
Therefore, the central challenge for stable low-bit quantization is to construct a coordinate system in which rounding errors are less amplified by high-curvature directions.
\end{mdframed}
In this paper, we propose Hessian Robust Quantization (HeRo-Q), a curvature-aware PTQ framework. Instead of directly minimizing quantization error, HeRo-Q performs quantization in a Hessian-conditioned coordinate system. Diagonal smoothing suppresses sharp effective-curvature directions, while learnable orthogonal rotation redistributes quantization noise away from sensitive directions. By decoupling these geometric roles, HeRo-Q effectively lowers the model's worst-case sensitivity to quantization perturbations. Importantly, HeRo-Q requires no changes to model architecture or activation functions—it introduces only a lightweight, learnable linear transformation on weights, making it fully compatible with existing PTQ pipelines for fast and effective quantization, shown in Figure \ref{fig:heroq_method}.
\begin{figure*}[t]

    \centering
    \begin{subfigure}[b]{0.32\textwidth}
        \centering
        \includegraphics[width=\linewidth]{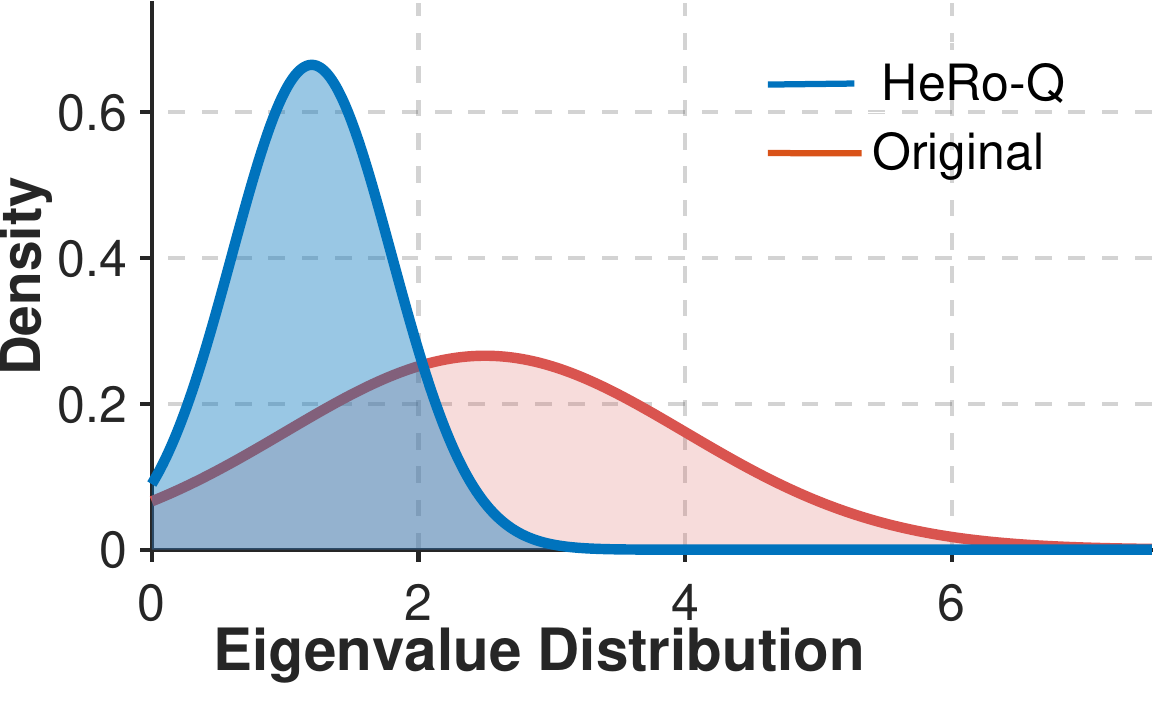}
        \caption{Effective curvature reshaping}
        \label{fig:heroq_method}
    \end{subfigure}
    \hfill
    \begin{subfigure}[b]{0.25\textwidth}
        \centering
        \includegraphics[width=\linewidth]{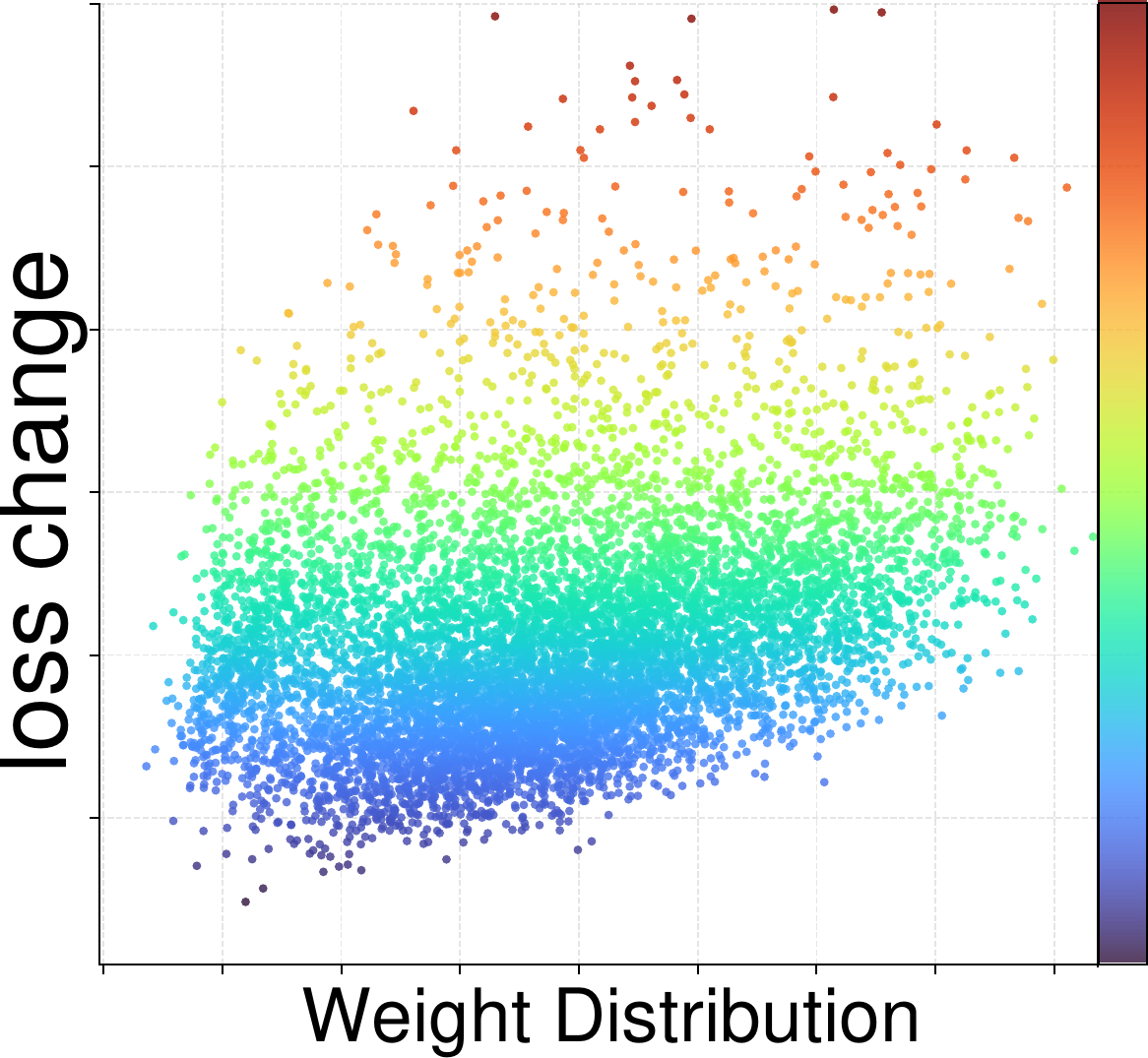}
        \caption{Low-error, high-loss}
        \label{fig:low_error_high_loss}
    \end{subfigure}
    \hfill
    \begin{subfigure}[b]{0.32\textwidth}
        \centering
        \includegraphics[width=\linewidth]{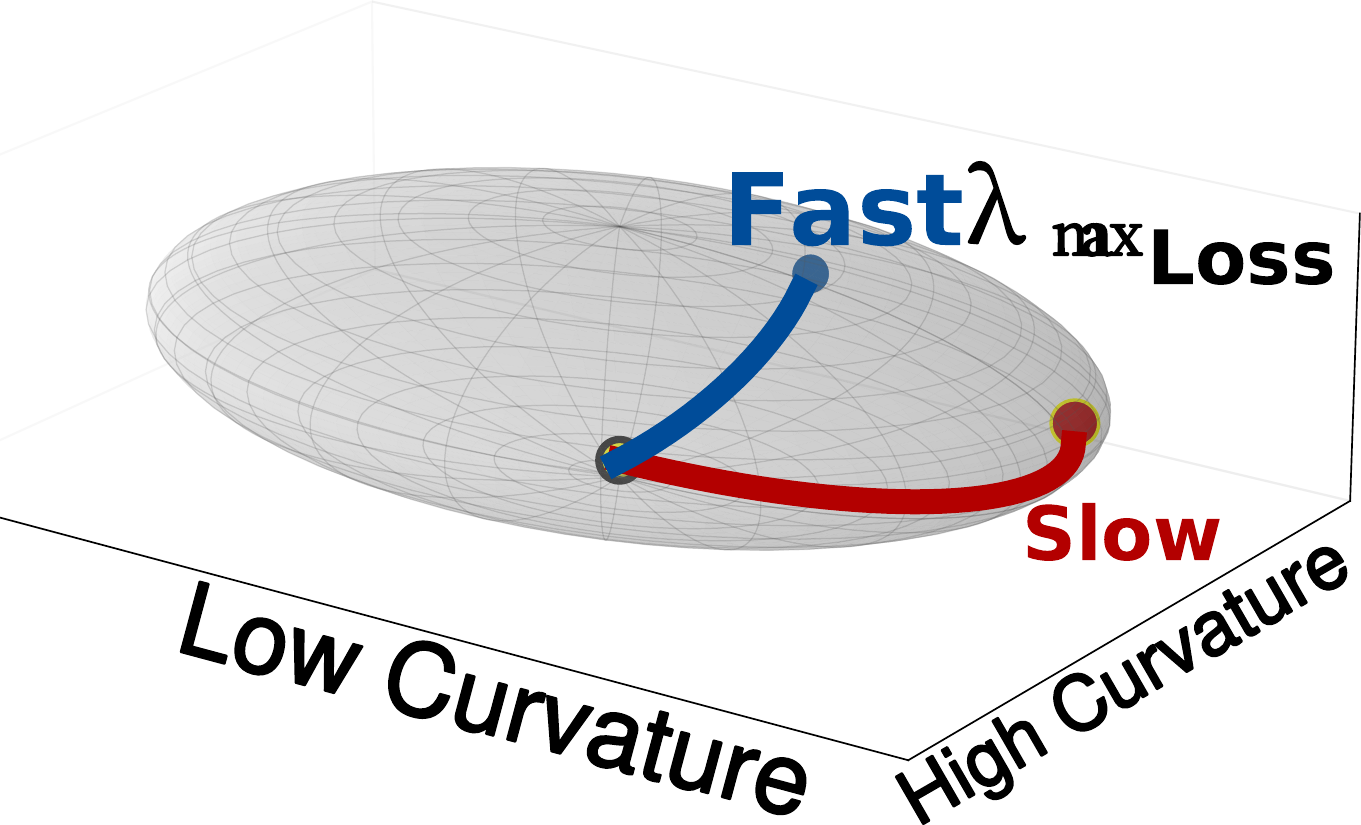}
        \caption{Local geometry of the loss landscape}
        \label{fig:loss_landscape_geometry}
    \end{subfigure}

    \caption{
    (a) HeRo-Q improves quantization robustness by reshaping the effective Hessian geometry in the quantization coordinate system, rather than directly modifying the original loss landscape.
    (b) Small quantization error does not necessarily imply small loss degradation: if the perturbation is aligned with high-curvature directions, the loss can increase substantially.
    \textbf{(c)} The local loss landscape around a converged solution can be approximated by a hyper-ellipsoid. Perturbations along the short axis (high curvature, large Hessian eigenvalue) induce much larger loss increases than perturbations of the same magnitude along the long axis (low curvature).}
    \label{fig:motivation_overview}
\end{figure*}

Our contributions are as follows:
(1) We establish a theoretical connection between quantization robustness and curvature-weighted loss from the perspective of loss landscape geometry, and propose the HeRo-Q algorithm accordingly.
(2) HeRo-Q applies a lightweight learnable transformation before quantization, incurring negligible additional inference overhead after offline fusion;
(3) In experiments, HeRo-Q achieves the best or near-best performance across most evaluated settings compared with state-of-the-art PTQ methods (e.g., GPTQ, AWQ, SpinQuant) on both Llama and Qwen model families. 


\section{Related Work}
\begin{figure*}[t]
    \centering
    \includegraphics[width=.95\linewidth]{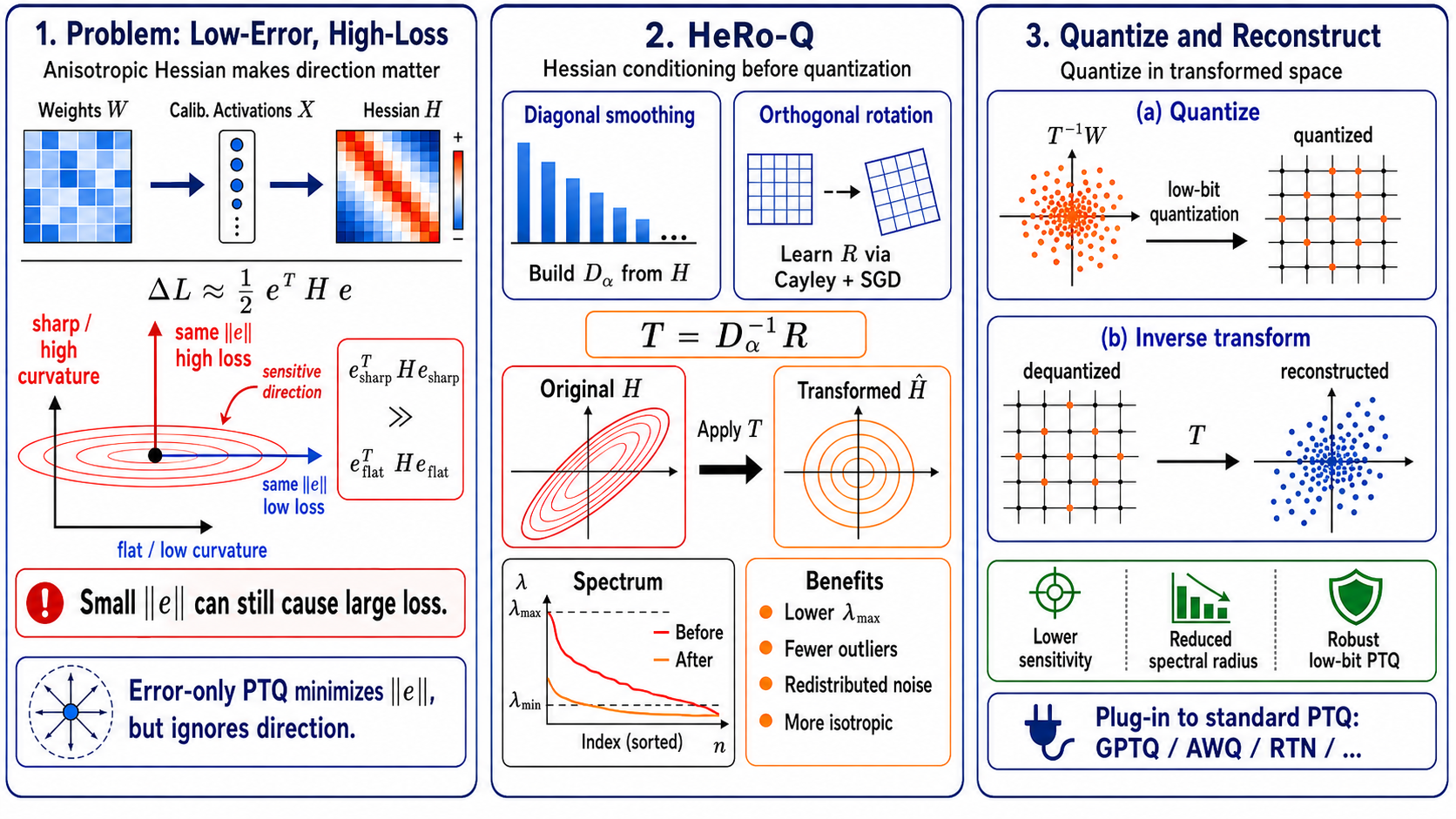}
    \caption{
    Robust quantization via Hessian-conditioned transformation.
    HeRo-Q first estimates the layer-wise curvature structure and identifies directions that are highly sensitive to quantization noise.
    It then searches over smoothing parameters \(\alpha\) to construct a diagonal scaling matrix \(\mathbf{D}_\alpha\), and learns an orthogonal rotation matrix \(\mathbf{R}\) via the Cayley transform.
    Together, they form a transformation \(\mathbf{T}=\mathbf{D}_\alpha^{-1}\mathbf{R}\), which defines a Hessian-conditioned coordinate system for quantization.
    Quantization is performed in this transformed space, and the quantized weights are mapped back to the original weight coordinate for inference.
    }
    \label{fig:placeholder}
\end{figure*}


Representative PTQ methods mainly focus on reducing the numerical discrepancy between full-precision and quantized models. AdaRound~\citep{DBLP:conf/icml/NagelABLB20} learns adaptive rounding decisions by minimizing layer-wise reconstruction error, while GPTQ~\citep{GPTQ} improves weight-only quantization through second-order reconstruction. AWQ~\citep{AWQ} preserves salient weights via activation-aware scaling, OmniQuant~\citep{DBLP:conf/iclr/ShaoC0XZLZ00024} jointly optimizes clipping thresholds and equivalent transformations, and QuIP~\citep{DBLP:conf/nips/CheeCKS23} introduces incoherence to stabilize extremely low-bit quantization. Recent studies further explore alternative quantization formulations and adaptive compression strategies~\citep{zhang2024fp,zhanglod,zhang2024compression,zhang2025rethinking,zhang2025moqe,zhang2026calmckaguidedadaptivelayerwise}.

Transformation-based PTQ improves quantization by changing the representation of weights or activations before quantization. SmoothQuant~\citep{SmoothQuant} mitigates activation outliers through channel-wise scaling, while SpinQuant~\citep{spinquant} learns orthogonal rotations to obtain more quantization-friendly coordinates. FlatQuant \cite{DBLP:conf/icml/SunLBBZLHY0Y0LY25} and OSTQuant \cite{DBLP:conf/iclr/0010CYCXYXY0Z25} further improve distribution flatness or quantization-space utilization through affine, orthogonal, or scaling transformations. These methods show that proper transformations can improve low-bit robustness, but they mainly seek distribution flatness or coordinate quantizability.

Mixed-precision and sensitivity-aware quantization methods allocate bit-widths according to accuracy, hardware cost, or layer sensitivity. AutoQ~\citep{lou2019autoq} and HAQ~\citep{wang2019haq} formulate bit-width assignment as a search problem, while other methods estimate sensitivity using KL divergence~\citep{cai2020zeroq,xu2024ptmq}, Hessian-based curvature~\citep{dong2020hawq,yao2021hawq}, Gauss--Newton approximations~\citep{chen2021towards}, or learned importance scores~\citep{tang2022mixed}. These methods mainly decide how much precision to allocate to different layers or parameters, whereas HeRo-Q focuses on designing a Hessian-conditioned transformation for fixed low-bit quantization.

The key distinction of HeRo-Q is not the mere use of smoothing or rotation, but the curvature-aware objective that guides their design. Existing PTQ methods mainly reduce quantization error or seek quantization-friendly coordinates. HeRo-Q instead targets the curvature-weighted loss consequence of quantization error, explicitly considering how perturbations are amplified by high-curvature Hessian directions. This reformulates transformation-based PTQ as a curvature-sensitive robustness problem.
\section{Methodology}

\subsection{Notation}
In this section, we list all notations in our paper, including the Appendix.
\begin{mdframed}[
    leftline=true,
    rightline=false,
    topline=false,
    bottomline=false,
    linewidth=2.5pt,
    linecolor=magenta!50,
    backgroundcolor=gray!10,
    innerleftmargin=8pt,
    innerrightmargin=8pt,
    innertopmargin=6pt,
    innerbottommargin=6pt
]
\small

\noindent
\begin{minipage}[t]{0.47\linewidth}
\vspace{0pt}
\begin{itemize}[leftmargin=1.2em, itemsep=1pt, topsep=0pt, parsep=0pt]
    \item $\mathbf{H}$: Hessian in original coordinate.
    \item $\tilde{\mathbf{H}}$: Effective Hessian in quantization coordinate.
    \item $\mathbf{T}$: Quantization-coordinate transform.
    \item $\lambda_{\max}$: Largest eigenvalue.
    \item $\alpha$: Smoothing hyperparameter.
    \item $\Delta \mathcal{L}$: Quantization-induced loss degradation.
\end{itemize}
\end{minipage}
\hfill
\begin{minipage}[t]{0.47\linewidth}
\vspace{0pt}
\begin{itemize}[leftmargin=1.2em, itemsep=1pt, topsep=0pt, parsep=0pt]
    \item $\mathbf{w}$: Pre-trained weights.
    \item $\tilde{\mathbf{w}}$: Transformed-coordinate weights.
    \item $\hat{\mathbf{w}}$: Quantized weights. $\mathbf{e}=\hat{\mathbf{w}}-\mathbf{w}$.
    \item $\mathbf{e}$: Original-coordinate quantization error.
    \item $\tilde{\mathbf{e}}$: Transformed-coordinate quantization error
    \item $\mathcal{B}(\alpha)$: Spectral error bound.
    \item $\mathbf{R}, \mathbf{D}_{\alpha}$: Rotation and smoothing matrices.
\end{itemize}
\end{minipage}

\end{mdframed}

\subsection{Motivation}
\label{sec:motivation}



We establish the relationship among the spectrum of the Hessian matrix, the noise introduced by quantization, and the resulting change in the loss function through the following Theorem \ref{the:loss_bound}. The proof is shown in  Appendix \ref{app:proof_theorem_3_1}.

We estimate the local curvature using the activation-induced Hessian approximation commonly adopted in layer-wise PTQ reconstruction. We introduce the smoothing coefficient \(\alpha\) to control the strength of diagonal preconditioning. It acts as a tuning knob that balances effective-curvature suppression and clipping error.

\begin{theorem}[Spectral Surrogate Bound]
\label{the:loss_bound}
By introducing a smoothing hyperparameter $\alpha$ (defined in Eq.~\ref{a_define}) to regulate the Hessian spectral distribution, the loss variation induced by quantization admits the following \emph{quadratic surrogate upper bound}:
\begin{equation}
\Delta \mathcal{L}_{\mathrm{quad}}
\leq
\frac{1}{2}
\lambda_{\max}\!\left(\tilde{\mathbf{H}}\right)
\left\|\tilde{\mathbf{e}}\right\|_2^2
\leq
\frac{1}{2}
\left\|\tilde{\mathbf{e}}\right\|_2^2
\cdot
\mathcal{B}(\alpha),
\label{eq:curvature_weighted_surrogate}
\end{equation}
where 
\(\tilde{\mathbf{H}}=\mathbf{T}^{\top}\mathbf{H}\mathbf{T}\) 
is the effective Hessian in the quantization coordinate system, and
\begin{equation}
\mathcal{B}(\alpha)
=
\max_k
\left(
\left|(\tilde{\mathbf{H}})_{kk}\right|
+
\sum_{j\neq k}
\left|(\tilde{\mathbf{H}})_{kj}\right|
\right).
\label{eq:spectral_surrogate_bound}
\end{equation}
Theorem~\ref{the:loss_bound} decomposes the quantization-induced loss into curvature amplification
$\lambda_{\max}(\tilde{\mathbf H})$ and transformed-space error
$\|\tilde{\mathbf{e}}\|_2^2$. This motivates HeRo-Q to construct a Hessian-conditioned quantization coordinate system: diagonal smoothing reduces curvature amplification, while orthogonal rotation reduces the alignment between transformed-space rounding errors and high-curvature directions.
\end{theorem}


In LLMs such as Llama-3.2-3B, the layer-wise curvature spectrum is often highly anisotropic, with a few sharp directions dominating the curvature response. According to Theorem~\ref{the:loss_bound}, quantization-induced degradation depends not only on the transformed-space error magnitude $\|\tilde{\mathbf{e}}\|_2$, but also on the spectral concentration of the effective Hessian. Thus, minimizing the Euclidean error alone can be insufficient when rounding errors are amplified by sharp effective-curvature directions, motivating HeRo-Q to jointly control curvature amplification and noise--curvature alignment in the quantization coordinate system.

\subsection{HeRo-Q Algorithm}
\label{sec:theory}

To address the instability induced by sharp effective-curvature directions, we propose HeRo-Q, a Hessian-conditioned PTQ framework, shown in Algorithm \ref{alg:hroq}. The overall workflow is shown in Figure~\ref{fig:placeholder}. 
The core idea is to construct a quantization coordinate system in which low-bit rounding errors are less amplified by local curvature.

In our algorithm, let the transformed equivalent Hessian be $\tilde{\mathbf{H}} \triangleq \mathbf{T}^\top \mathbf{H} \mathbf{T}$. Guided by Theorem \ref{the:loss_bound}, which decomposes the quantization-induced loss perturbation into curvature amplification and transformed-space quantization error, we design $\mathbf{T}$ to reduce their curvature-weighted effect. To this end, we decompose curvature-aware robustness into two complementary operations: \textit{diagonal smoothing} to reduce curvature amplification, and \textit{orthogonal rotation} to reduce the alignment between rounding errors and high-curvature directions. Specifically, we define:
\begin{equation} 
\label{eq:transform_def}
    \mathbf{T} \triangleq \mathbf{D}_{\alpha}^{-1}\mathbf{R},
\end{equation}
where $\mathbf{R} \in \mathbb{R}^{d \times d}$ is an orthogonal rotation matrix ($\mathbf{R}^\top \mathbf{R} = \mathbf{I}$), and $\mathbf{D}_{\alpha}$ is a diagonal smoothing matrix constructed from a layer-wise Hessian proxy:
\begin{equation}
\label{a_define}
\mathbf{D}_{\alpha} = \mathrm{diag}\left(|h_1|^{\alpha/2}, |h_2|^{\alpha/2}, \dots, |h_d|^{\alpha/2}\right),
\quad \mathbf{h}=\mathrm{diag}(\mathbf{H}_{\mathrm{proxy}}),\quad
\mathbf{H}_{\mathrm{proxy}}=\mathbf{X}^{\top}\mathbf{X},
\end{equation}
with $\alpha \in [0, 1]$ acting as a tunable hyperparameter that controls the aggressiveness of outlier suppression—effectively serving as a \textit{spectral equalizer}. The theoretical justification for the smoothing component is provided by the following result:

\begin{theorem}[Strict Reduction of Spectral Error Bound]
\label{thm:spectral_compression}
Under a normalized long-tail Hessian-proxy assumption, i.e.,
$\displaystyle \frac{\max_i |h_i|}{\operatorname{median}_i |h_i|} \gg 1$
\cite{sagun2017empirical,sankar2021deeper,papyan2018full,DBLP:conf/nips/ZhangCDL0L24,DBLP:conf/icml/0001GC00XX25},
there exists $\alpha^\ast \in (0,1)$ such that
\begin{equation}
B(\alpha^\ast) < B(0).
\end{equation}
This result implies that Hessian-guided smoothing can strictly tighten the curvature-amplification bound compared to the unsmoothed baseline ($\alpha = 0$), providing the theoretical motivation for the smoothing component of HeRo-Q.
\end{theorem}

This theorem confirms that attenuating dominant diagonal entries via smoothing strictly reduces the spectral error bound. Building on this geometric insight, HeRo-Q implements a practical algorithm that searches for the optimal $\alpha^*$ to reduce this bound, while simultaneously learning an orthogonal rotation $\mathbf{R}^*$ to reduce noise--curvature alignment. Hyperparameters are shown in the Appendix \ref{app:hyperparameters}.

HeRo-Q operates in three phases: First, it estimates the diagonal of the activation-induced Hessian approximation from calibration data to capture the channel-wise curvature sensitivity. Second, for each candidate smoothing exponent $\alpha$ in a predefined search grid $\mathcal{A}$, it constructs a diagonal smoothing matrix $\mathbf{D}_\alpha = \mathrm{diag}(|\mathbf{h}|^{\alpha/2})$ and combines it with a learnable orthogonal rotation matrix $\mathbf{R}$ to form a joint transformation $\mathbf{T}(\mathbf{R}) = \mathbf{D}_\alpha^{-1} \mathbf{R}$. The rotation matrix is parameterized via the Cayley transform to preserve orthogonality and is optimized through gradient descent to minimize the Frobenius norm reconstruction error $\|\mathbf{X}\mathbf{W} - \mathbf{X}\mathbf{W}_{\mathrm{sim}}\|_F^2$ between the original and quantized-simulated outputs. Finally, the algorithm selects the optimal pair $(\alpha^*, \mathbf{R}^*)$ that yields the lowest reconstruction error, applies the corresponding optimal transform $\mathbf{T}^* = \mathbf{D}_{\alpha^*}^{-1} \mathbf{R}^*$, performs quantization in the transformed coordinate system, and maps the quantized weights back to the original coordinate to obtain $\hat{\mathbf{W}}$. The base quantizer can be instantiated with GPTQ, RTN, AWQ, or other standard PTQ methods.


\begin{algorithm}[t]
\small
\caption{HeRo-Q: Hessian-Robust Quantization}
\label{alg:hroq}
\begin{algorithmic}[1]
\STATE {\bfseries Input:} weights $\mathbf{W}$, calibration inputs $\mathbf{X}$, grid $\mathcal{A}$, bits $b$
\STATE {\bfseries Output:} quantized weights $\hat{\mathbf{W}}$, transform $(\mathbf{D}_{\alpha^*}, \mathbf{R}^*)$
\FOR{each linear layer}
    \STATE $\mathbf{H}_{p}\leftarrow \mathbf{X}^{\top}\mathbf{X}$, \quad
           $\mathbf{h}\leftarrow \operatorname{diag}(\mathbf{H}_{p})$, \quad
           $\mathcal{L}^*\leftarrow \infty$
    \FOR{$\alpha\in\mathcal{A}$}
        \STATE $\mathbf{D}_{\alpha}\leftarrow \operatorname{diag}(|\mathbf{h}|^{\alpha/2})$
        \STATE Optimize orthogonal $R_\alpha$ via Cayley-SGD
        \STATE $\mathbf{T}_{\alpha}\leftarrow \mathbf{D}_{\alpha}^{-1}\mathbf{R}_{\alpha}$, \quad
               $\tilde{\mathbf{H}}_{\alpha}\leftarrow \mathbf{T}_{\alpha}^{\top}\mathbf{H}_{p}\mathbf{T}_{\alpha}$
        \STATE $\tilde{\mathbf{W}}_{q}\leftarrow 
        \operatorname{Quant}(\mathbf{T}_{\alpha}^{-1}\mathbf{W}, \tilde{\mathbf{H}}_{\alpha}, b)$
        \STATE $\hat{\mathbf{W}}_{\alpha}\leftarrow \mathbf{T}_{\alpha}\tilde{\mathbf{W}}_{q}$
        \STATE $\mathcal{L}_{\alpha}\leftarrow
        \|\mathbf{X}\mathbf{W}-\mathbf{X}\hat{\mathbf{W}}_{\alpha}\|_{F}^{2}$
        \IF{$\mathcal{L}_{\alpha}<\mathcal{L}^*$}
            \STATE $\mathcal{L}^*\leftarrow \mathcal{L}_{\alpha}$, \quad
                   $(\alpha^*,\mathbf{R}^*)\leftarrow(\alpha,\mathbf{R}_{\alpha})$
        \ENDIF
    \ENDFOR
    \STATE $\hat{\mathbf{W}}\leftarrow 
    \mathbf{T}^* \operatorname{Quant}((\mathbf{T}^*)^{-1}\mathbf{W},
    (\mathbf{T}^*)^{\top}\mathbf{H}_{p}\mathbf{T}^*, b)$, where
    $\mathbf{T}^*=\mathbf{D}_{\alpha^*}^{-1}\mathbf{R}^*$
\ENDFOR
\end{algorithmic}
\end{algorithm}

\subsection{Inference Efficiency}
\label{sec:inference_efficiency}

HeRo-Q employs a hybrid inference strategy that combines offline parameter
fusion with optimized online computational kernels.

First, for the vast majority of fusible linear layers in the model, HeRo-Q introduces no additional runtime operators after offline re-parameterization. We absorb the inverse of the smoothing matrix $\mathbf{D}_\alpha^{-1}$ directly into normalization layers (LayerNorm \cite{ba2016layer} or RMSNorm \cite{DBLP:conf/nips/ZhangS19a}) when applicable. This eliminates additional runtime multiplications for the fused linear path. For the dense rotation matrix $\mathbf{R}$ learned via Cayley-SGD \cite{li2020efficient}, we fuse it offline into the adjacent weights ($\mathbf{W}_{prev} \leftarrow \mathbf{W}_{prev}\mathbf{R}$ and $\mathbf{W}_{curr} \leftarrow \mathbf{R}^\top \mathbf{W}_{curr}$), keeping the computational graph architecture unchanged.

Second, for components where weight fusion is inapplicable, such as KV cache
quantization or specific online activation processing, we follow the approach of SpinQuant~\citep{spinquant} and QuaRot \cite{ashkboos2024quarot} by adopting online Hadamard rotation. In this scenario, we utilize the Fast Walsh--Hadamard Transform (FWHT) kernel for efficient computation during inference~\cite{DBLP:journals/siamcomp/AilonC09}. Therefore, HeRo-Q introduces no additional runtime operators for fusible linear layers and only negligible measured inference overhead for non-fusible layer components. Complexity analysis is provided in Appendix~\ref{app:efficiency_analysispp}.
 
\section{Experiments}

We evaluate HeRo-Q from three aspects. First, we compare its downstream performance with strong PTQ baselines under multiple quantization regimes, including W4A4, W4A8, W4A16, and W3A16. Second, we analyze the Hessian spectrum to verify whether HeRo-Q indeed reduces sensitivity to high curvature. Third, we report inference throughput to examine its practical deployment cost. The code is provided in Appendix~\ref{code}.

\subsection{Experimental Setup}

\textbf{Models and tasks.}
We evaluate HeRo-Q on dense LLMs, vision models, multimodal models, and MoE models. The dense LLMs include Llama-3.1-8B \cite{DBLP:journals/corr/abs-2407-21783}, Llama-3.2-1B/3B \cite{llama32_modelcard}, and Qwen2.5-3B/7B \cite{DBLP:journals/corr/abs-2412-15115}. We report perplexity on WikiText2 \cite{DBLP:conf/iclr/MerityX0S17} and C4 \cite{DBLP:journals/jmlr/RaffelSRLNMZLL20} with sequence length 2048, and evaluate downstream accuracy on MMLU \cite{DBLP:conf/iclr/HendrycksBBZMSS21}, GSM8K \cite{DBLP:journals/corr/abs-2110-14168}, and HellaSwag \cite{DBLP:conf/acl/ZellersHBFC19}. To test cross-architecture generality, we further evaluate ViT/DeiT~\citep{DBLP:conf/iclr/DosovitskiyB0WZ21,DBLP:conf/icml/TouvronCDMSJ21}, LLaVA-OneVision-7B \cite{DBLP:journals/tmlr/0080ZGZ00ZZL0L25}, DeepSeek-MoE-16B-chat \cite{DBLP:conf/acl/DaiDZXGCLZYWXLH24}, and Qwen-MoE-14B-chat \cite{bai2023qwen}.

\textbf{Baselines and settings.}
We compare with representative PTQ methods, including RTN \cite{DBLP:journals/corr/abs-2208-07339}, GPTQ \cite{GPTQ}, AWQ \cite{AWQ}, SmoothQuant \cite{SmoothQuant}, SpinQuant \cite{spinquant}, OmniQuant \cite{DBLP:conf/iclr/ShaoC0XZLZ00024}, and FlatQuant \cite{DBLP:conf/icml/SunLBBZLHY0Y0LY25}. For vision, multimodal, and MoE models, we additionally include task-specific baselines such as FQ-ViT \cite{DBLP:conf/ijcai/LinZSLZ22}, PTQ4ViT \cite{DBLP:conf/eccv/YuanXCWS22}, APQ-ViT \cite{DBLP:conf/mm/DingQYCLWL22}, RepQ-ViT \cite{DBLP:conf/iccv/LiXYG23}, MBQ \cite{DBLP:conf/cvpr/LiHNLHJLYRDYY025}, and MoEQuant+ \cite{DBLP:conf/icml/Chen0YXXYZY25}. All methods use the same calibration data and evaluation protocol, with official settings unless otherwise specified.

\textbf{Implementation details.}
For HeRo-Q, the smoothing parameter $\alpha$ is selected layer-wise, and GPTQ is used as the default base quantizer unless otherwise specified. Calibration and evaluation data are strictly separated. For downstream LLM evaluation, we follow standard few-shot protocols: MMLU 5-shot, GSM8K 8-shot, and HellaSwag 10-shot \cite{eval-harness}.

\subsection{Downstream Task Evaluation}

\begin{table*}[t]
\centering
\renewcommand{\arraystretch}{1.08}
\setlength{\tabcolsep}{1.8pt}
\tiny
\resizebox{\textwidth}{!}{
\begin{tabular}{l|cccc|cccc||cccc|cccc}
\toprule
\multirow{2}{*}{\textbf{Method}}
& \multicolumn{4}{c|}{\textbf{Llama-3.1-8B (W4A16)}}
& \multicolumn{4}{c||}{\textbf{Qwen2.5-7B (W4A16)}}
& \multicolumn{4}{c|}{\textbf{Llama-3.1-8B (W3A16)}}
& \multicolumn{4}{c}{\textbf{Llama-3.2-1B (W3A16)}} \\
\cline{2-17}
& \textbf{W2$\downarrow$} & \textbf{C4$\downarrow$} & \textbf{MMLU$\uparrow$} & \textbf{GSM8K$\uparrow$}
& \textbf{W2$\downarrow$} & \textbf{C4$\downarrow$} & \textbf{MMLU$\uparrow$} & \textbf{GSM8K$\uparrow$}
& \textbf{W2$\downarrow$} & \textbf{C4$\downarrow$} & \textbf{MMLU$\uparrow$} & \textbf{GSM8K$\uparrow$}
& \textbf{W2$\downarrow$} & \textbf{C4$\downarrow$} & \textbf{MMLU$\uparrow$} & \textbf{GSM8K$\uparrow$} \\
\midrule

FP16
& 6.41 & 12.28 & 67.31 & 77.17
& 6.66 & 13.26 & 71.48 & 87.96
& 6.41 & 12.28 & 67.31 & 77.17
& 11.41 & 22.48 & 46.38 & 36.44 \\

GPTQ
& 6.65 & 12.55 & 65.80 & 74.27
& 6.95 & 13.85 & 69.80 & 85.23
& 20.13 & 28.94 & 25.12 & 26.36
& 41.30 & 65.80 & 14.20 & 15.19 \\

AWQ
& 6.60 & 12.56 & 66.14 & 74.52
& 7.14 & 14.21 & 69.68 & 82.33
& 15.87 & 24.15 & 39.50 & 46.21
& 12.23 & 23.97 & 40.23 & 26.44 \\

OmniQuant
& 6.48 & 12.35 & 66.80 & 75.83
& 6.75 & 13.40 & 70.90 & 86.82
& 9.45 & 15.80 & 59.80 & 63.59
& 17.50 & 32.10 & 36.50 & 21.48 \\

SpinQuant
& 6.46 & 12.32 & 66.85 & 76.17
& 6.72 & 13.31 & 71.10 & 87.14
& 8.85 & 14.95 & 60.50 & 66.22
& 16.20 & 30.05 & 39.10 & 22.80 \\

FlatQuant
& 6.46 & 12.32 & 66.93 & 76.32
& 6.71 & 13.35 & 71.17 & 87.22
& 8.77 & 14.55 & 62.56 & 68.38
& 16.11 & 27.47 & 40.18 & 24.97 \\

\rowcolor{blue!8}
\textbf{HeRo-Q}
& \textbf{6.43} & \textbf{12.30} & \textbf{67.10} & \textbf{76.85}
& \textbf{6.68} & \textbf{13.28} & \textbf{71.35} & \textbf{87.50}
& \textbf{8.02} & \textbf{13.98} & \textbf{63.80} & \textbf{70.15}
& \textbf{14.05} & \textbf{25.80} & \textbf{42.50} & \textbf{29.10} \\

\midrule[0.9pt]

\multirow{2}{*}{\textbf{Method}}
& \multicolumn{4}{c|}{\textbf{Llama-3.1-8B (W4A4)}}
& \multicolumn{4}{c||}{\textbf{Llama-3.2-3B (W4A4)}}
& \multicolumn{4}{c|}{\textbf{Llama-3.1-8B (W4A8)}}
& \multicolumn{4}{c}{\textbf{Qwen2.5-7B (W4A8)}} \\
\cline{2-17}
& \textbf{W2$\downarrow$} & \textbf{C4$\downarrow$} & \textbf{MMLU$\uparrow$} & \textbf{GSM8K$\uparrow$}
& \textbf{W2$\downarrow$} & \textbf{C4$\downarrow$} & \textbf{MMLU$\uparrow$} & \textbf{GSM8K$\uparrow$}
& \textbf{W2$\downarrow$} & \textbf{C4$\downarrow$} & \textbf{MMLU$\uparrow$} & \textbf{GSM8K$\uparrow$}
& \textbf{W2$\downarrow$} & \textbf{C4$\downarrow$} & \textbf{MMLU$\uparrow$} & \textbf{GSM8K$\uparrow$} \\
\midrule

FP16
& 6.41 & 12.28 & 67.31 & 77.17
& 9.53 & 17.24 & 61.44 & 68.86
& 6.41 & 12.28 & 67.31 & 77.17
& 6.66 & 13.26 & 71.48 & 87.96 \\

AWQ
& 346 & 642.4 & 36.45 & 30.76
& 1000+ & 1000+ & 22.97 & 16.88
& 8.12 & 13.56 & 64.41 & 72.94
& 7.14 & 14.21 & 69.68 & 82.33 \\

GPTQ
& 1000+ & 1000+ & 34.55 & 28.79 
& 1000+ & 1000+ & 23.12 & 13.17
& 8.81 & 14.12 & 60.54 & 71.89
& 7.32 & 14.53 & 69.35 & 81.67 \\

SmoothQuant
& 429 & 745 & 40.18 & 37.81
& 1000+ & 1000+ & 34.74 & 20.76
& 7.32 & 13.64 & 64.55 & 72.55
& 7.09 & 14.06 & 69.79 & 83.47 \\

SpinQuant
& 12.14 & 17.46 & 43.76 & 39.44
& 15.31 & 25.41 & 36.98 & 32.91
& 7.48 & 13.39 & 64.98 & 73.56
& 6.97 & 13.89 & 70.14 & 84.97 \\

FlatQuant
& 11.17 & 15.94 & 44.39 & 41.01
& 14.65 & 23.33 & 38.25 & 34.86
& 7.31 & 13.11 & 65.32 & 74.19
& 6.89 & 13.59 & 70.28 & 85.37 \\

\rowcolor{blue!8}
\textbf{HeRo-Q}
& \textbf{10.29} & \textbf{14.84} & \textbf{45.31}  & \textbf{42.33}
& \textbf{13.18} & \textbf{20.14} & \textbf{41.10} & \textbf{37.15}
& \textbf{6.96} & \textbf{12.94} & \textbf{66.31} & \textbf{75.26}
& \textbf{6.78} & \textbf{13.43} & \textbf{71.02} & \textbf{86.12} \\
\bottomrule
\end{tabular}
}
\caption{Compact comparison across bit-width regimes. W2 and C4 denote perplexity on WikiText-2 and C4; MMLU and GSM8K denote accuracy scores. Lower perplexity is better, while higher accuracy is better.}
\label{tab:main_bitwidth}
\end{table*}
\paragraph{Results on dense LLMs.}
Table \ref{tab:main_bitwidth} reports the downstream performance of HeRo-Q under multiple quantization regimes, including W4A16, W4A8, W4A4, and W3A16. Overall, HeRo-Q consistently achieves the best or near-best results across both language modeling and reasoning benchmarks. This trend holds across different model families, including Llama and Qwen, suggesting that the proposed Hessian-conditioned transformation is not tied to a specific architecture. Under the standard W4A16 setting, HeRo-Q nearly preserves FP16-level performance. For example, on Llama-3.1-8B, HeRo-Q achieves 6.43 WikiText-2 PPL and 76.85 GSM8K accuracy, which is close to the FP16 results of 6.41 and 77.17. On Qwen2.5-7B, HeRo-Q obtains 71.35 MMLU and 87.50 GSM8K, outperforming GPTQ, AWQ, SpinQuant, and FlatQuant.

Under the practical W4A8 setting, HeRo-Q also shows consistent advantages. On Llama-3.1-8B, it reduces WikiText-2 PPL to 6.96 and improves GSM8K accuracy to 75.26. On Qwen2.5-7B, HeRo-Q reaches 6.78 WikiText-2 PPL and 86.12 GSM8K accuracy, outperforming other transform-based and reconstruction-based baselines. Since W4A8 is commonly used in deployment-oriented quantization, these results indicate that HeRo-Q improves the accuracy-efficiency trade-off not only in extreme cases, but also in realistic low-bit inference scenarios.

The advantage becomes more pronounced under aggressive quantization. In the W3A16 setting on Llama-3.1-8B, GPTQ suffers severe degradation, with WikiText-2 PPL increasing to 20.13 and GSM8K accuracy dropping to 26.36. AWQ also degrades noticeably, obtaining 15.87 WikiText-2 PPL and 46.21 GSM8K accuracy. In contrast, HeRo-Q achieves 8.02 WikiText-2 PPL and 70.15 GSM8K accuracy, substantially narrowing the gap to FP16. Similar robustness is observed under W4A4, where several baselines collapse with extremely large perplexity. For instance, GPTQ reaches 1000+ PPL on Llama-3.1-8B, while HeRo-Q still maintains 10.29 PPL and 42.33 GSM8K accuracy.

\paragraph{Generalization to vision and multimodal models.}
Table \ref{tab:vision_multimodal} further evaluates whether HeRo-Q generalizes beyond dense LLMs. On ViT/DeiT models, HeRo-Q consistently outperforms prior ViT-oriented PTQ methods under both W4A4 and W6A6 settings. For example, under W4A4, HeRo-Q achieves 70.12 on ViT-B, 70.01 on DeiT-S, and 77.33 on DeiT-B, outperforming RepQ-ViT across all three models. Under W6A6, HeRo-Q also remains close to the FP16 accuracy, reaching 84.10 on ViT-B and 81.66 on DeiT-B. On the multimodal LLaVA-OneVision-7B benchmark, HeRo-Q obtains the best overall performance among quantized methods, achieving 43.12 on MMMU and 74.31 on TextVQA. These results indicate that Hessian-conditioned transformation is not limited to language-only models, but also improves quantization robustness in vision and vision-language architectures.
\begin{table*}[t] 
\centering 
\renewcommand{\arraystretch}{1.05} 
\setlength{\tabcolsep}{2.4pt} 
\tiny 
\resizebox{\textwidth}{!}{ \begin{tabular}{l|ccc|ccc||lcccc} \toprule \multirow{2}{*}{\textbf{Method}} & \multicolumn{3}{c|}{\textbf{ViT/DeiT (W4A4)}} & \multicolumn{3}{c||}{\textbf{ViT/DeiT (W6A6)}} & \multirow{2}{*}{\textbf{Method}} & \multicolumn{4}{c}{\textbf{LLaVA-OneVision-7B (W3A16)}} \\ \cline{2-7}\cline{9-12} & \textbf{ViT-B$\uparrow$} & \textbf{DeiT-S$\uparrow$} & \textbf{DeiT-B$\uparrow$} & \textbf{ViT-B$\uparrow$} & \textbf{DeiT-S$\uparrow$} & \textbf{DeiT-B$\uparrow$} & & \textbf{MMMU$\uparrow$} & \textbf{VizWiz$\uparrow$} & \textbf{SQA$\uparrow$} & \textbf{TextVQA$\uparrow$} \\ \midrule FP16 & 84.54 & 79.85 & 81.80 & 84.54 & 79.85 & 81.80 & FP16 & 46.09 & 60.42 & 85.42 & 76.16 \\ FQ-ViT & 0.10 & 0.10 & 0.10 & 0.10 & 45.51 & 64.63 & RTN & 34.71 & 59.26 & 82.28 & 60.97 \\ PTQ4ViT & 30.69 & 34.08 & 64.39 & 81.65 & 76.28 & 80.25 & GPTQ & 41.74 & 56.43 & 82.49 & 71.33 \\ APQ-ViT & 41.41 & 43.55 & 67.48 & 82.21 & 77.76 & 80.42 & AWQ & 36.68 & 58.57 & 83.26 & 73.04 \\ RepQ-ViT & 68.48 & 69.03 & 75.61 & 83.62 & 78.90 & 81.27 & MBQ & 42.03 & 60.42 & 85.02 & 73.37 \\ \rowcolor{blue!8} \textbf{HeRo-Q} & \textbf{70.12} & \textbf{70.01} & \textbf{77.33} & \textbf{84.10} & \textbf{79.31} & \textbf{81.66} & \textbf{HeRo-Q} & \textbf{43.12} & \textbf{60.42} & \textbf{85.31} & \textbf{74.31} \\ \bottomrule \end{tabular} } \caption{Generalization of HeRo-Q on vision and multimodal models. } \label{tab:vision_multimodal} \end{table*}

\paragraph{Generalization to MoE models.}

Table \ref{tab:moe_full_results} evaluates HeRo-Q on MoE architectures, which are more challenging because different experts can exhibit heterogeneous activation and curvature patterns. On DeepSeek-MoE-16B-chat, HeRo-Q achieves 7.42 WikiText-2 PPL and 10.03 C4 PPL, both close to the FP16 results and better than most PTQ baselines. It also obtains strong downstream accuracy, including 47.88 on MMLU and 50.05 on GSM8K. On Qwen-MoE-14B-chat, HeRo-Q achieves 58.44 MMLU, 28.46 GSM8K, and 59.12 HellaSwag, consistently outperforming RTN, AWQ, and MoEQuant+ on most reasoning benchmarks. These results suggest that HeRo-Q remains effective even when quantization sensitivity varies across experts, supporting its role as a general Hessian-aware PTQ framework. More extensive evaluations are provided in Appendix~\ref{app:additional_experiments}. 
\begin{table*}[t]
\centering
\renewcommand{\arraystretch}{1.05}
\setlength{\tabcolsep}{2.4pt}
\scriptsize
\resizebox{\textwidth}{!}{
\begin{tabular}{l|cccccccc||cccccccc}
\toprule
\multirow{2}{*}{\textbf{Method}}
& \multicolumn{8}{c||}{\textbf{DeepSeek-MoE-16B-chat}}
& \multicolumn{8}{c}{\textbf{Qwen-MoE-14B-chat}} \\
\cline{2-17}
& \textbf{W2$\downarrow$} & \textbf{C4$\downarrow$} & \textbf{MMLU$\uparrow$} & \textbf{HE$\uparrow$} & \textbf{GSM8K$\uparrow$} & \textbf{HS$\uparrow$} & \textbf{OBQA$\uparrow$} & \textbf{MathQA$\uparrow$}
& \textbf{W2$\downarrow$} & \textbf{C4$\downarrow$} & \textbf{MMLU$\uparrow$} & \textbf{HE$\uparrow$} & \textbf{GSM8K$\uparrow$} & \textbf{HS$\uparrow$} & \textbf{OBQA$\uparrow$} & \textbf{MathQA$\uparrow$} \\
\midrule

FP16
& 7.35 & 9.96 & 48.90 & 24.39 & 54.28 & 60.69 & 33.40 & 34.27
& 8.07 & 9.74 & 59.00 & 21.34 & 30.71 & 59.33 & 31.00 & 34.91 \\

RTN
& 8.63 & 11.06 & 41.40 & 10.41 & 28.88 & 57.59 & 31.40 & 29.04
& 12.81 & 14.03 & 43.00 & 7.32 & 9.70 & 51.41 & 24.40 & 28.81 \\

AWQ
& 7.72 & 10.49 & 46.33 & 18.90 & 39.88 & 58.97 & \best{33.80} & 32.86
& 9.97 & 11.90 & 52.06 & 12.20 & 17.74 & 55.37 & 30.40 & 31.46 \\

GPTQ
& 7.55 & 10.24 & 46.60 & 13.41 & 47.08 & 59.64 & 33.20 & 32.76
& 8.38 & 10.78 & 57.30 & 15.24 & 26.08 & 58.72 & \best{31.40} & 34.17 \\

MoEQuant+
& 7.85 & 10.23 & 46.40 & 18.90 & 45.41 & 59.03 & 33.60 & 33.14
& 10.12 & 11.55 & 55.34 & 13.60 & 20.87 & 56.64 & 30.60 & 32.50 \\

\textbf{HeRo-Q}
& \best{7.42} & \best{10.03} & \best{47.88} & \best{21.12} & \best{50.05} & \best{60.48} & 33.40 & \best{33.87}
& \best{8.21} & \best{10.17} & \best{58.44} & \best{16.98} & \best{28.46} & \best{59.12} & 31.02 & \best{34.88} \\

\bottomrule
\end{tabular}
}
\caption{MoE quantization results on DeepSeek-MoE-16B-chat and Qwen-MoE-14B-chat. 
W2 and C4 denote WikiText-2 and C4 perplexity; HE, HS, and OBQA denote HumanEval, HellaSwag, and OpenBookQA, respectively.}
\label{tab:moe_full_results}
\end{table*}

\subsection{Spectral Analysis Experiment}

We analyze the effective Hessian eigenspectrum of Llama-3.2-3B before and after applying the proposed transformation. Figure~\ref{fig:hessian_spectrum} shows both the global eigenspectrum distribution and the layer-wise maximum effective eigenvalue. These results provide direct evidence that HeRo-Q improves quantization robustness by reshaping the effective curvature geometry rather than merely reducing reconstruction error.

\begin{figure}[t]
    \centering
    \includegraphics[width=.9\linewidth]{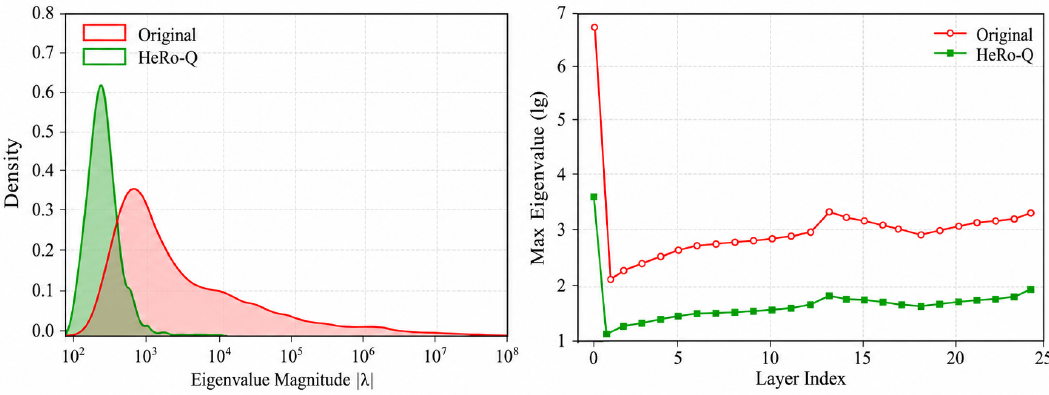} 
    \caption{(a) HeRo-Q compresses the long-tailed Hessian eigenspectrum and suppresses extreme eigenvalues.
    (b) HeRo-Q consistently reduces layer-wise maximum eigenvalues in log scale, indicating lower high-curvature sensitivity and improved robustness to quantization noise.
    }
    \label{fig:hessian_spectrum}
\end{figure}

Figure~\ref{fig:hessian_spectrum}(a) shows the density distribution of effective Hessian eigenvalue magnitudes. The baseline exhibits a clear long-tailed spectrum, where a non-negligible fraction of eigenvalues extends to extremely large magnitudes. Such spectral outliers correspond to high-curvature directions in the effective Hessian geometry. According to Theorem~3.1, perturbations along these directions dominate the worst-case quantization-induced loss increase, making the model highly sensitive to low-bit quantization noise. After applying HeRo-Q, the eigenspectrum becomes much more concentrated, and the heavy tail is significantly compressed. 

Figure~\ref{fig:hessian_spectrum}(b) shows that the baseline has sharp layer-wise
spectral peaks, indicating highly uneven quantization sensitivity across depth.
HeRo-Q consistently reduces these maximum effective eigenvalues and smooths
their variation, acting as a spectral equalizer under aggressive settings such
as W3A16. This compression mainly comes from Hessian-guided smoothing, while
rotation further improves robustness by reducing the alignment between
quantization noise and curvature rather than changing the spectrum itself.

\subsection{Computational Performance Evaluation}

 The inference throughput results in Table~\ref{tab:efficiency_alpha}(a) demonstrate the practical efficiency of weight-only quantization under the W4A16 setting on an NVIDIA A800 GPU. All quantized methods---GPTQ, SpinQuant, and HeRo-Q---deliver substantial speedups over the FP16 baseline. Notably, HeRo-Q achieves throughput on par with or slightly better than SpinQuant (e.g., 113.1 vs. 112.8 Tokens/s on Llama-3.1-8B), while matching the high efficiency of GPTQ. This indicates that the spectral preconditioning introduced by HeRo-Q does not compromise runtime performance. The negligible speed difference relative to other SOTA methods suggests that the orthogonal rotation and diagonal smoothing operations are efficiently fused during deployment, enabling HeRo-Q to maintain competitive inference latency while offering improved accuracy.

\subsection{Ablation Experiments}
\subsubsection{Hyperparameter Sensitivity Study}
\begin{table*}[t]
\centering
\renewcommand{\arraystretch}{1.00}
\setlength{\tabcolsep}{3.0pt}
\scriptsize

\begin{minipage}[t]{0.40\textwidth}
\centering
\textbf{(a) Inference Throughput}
\resizebox{\linewidth}{!}{
\begin{tabular}{lccc}
\toprule
\cellcolor{gray!10}\textbf{Method} 
& \cellcolor{gray!10}\textbf{Llama-3.2-1B} 
& \cellcolor{gray!10}\textbf{Llama-3.2-3B} 
& \cellcolor{gray!10}\textbf{Llama-3.1-8B} \\
\midrule
FP16      & 205.2 & 118.4 & 48.5 \\
GPTQ      & \best{\textbf{310.5}} & \best{\textbf{215.6}} & \best{\textbf{115.2}} \\
SpinQuant & 302.8 & 209.5 & 112.8 \\
HeRo-Q    & 303.1 & 210.0 & 113.1 \\
\bottomrule
\end{tabular}
}
\end{minipage}
\hfill
\begin{minipage}[t]{0.57\textwidth}
\centering
\textbf{(b) $\alpha$ Sensitivity}
\resizebox{\linewidth}{!}{
\begin{tabular}{lccccccccc}
\toprule
\textbf{Metric} & \textbf{0.1} & \textbf{0.2} & \cellcolor{blue!10}\textbf{0.3} & \textbf{0.4} & \textbf{0.5} & \textbf{0.6} & \textbf{0.7} & \textbf{0.8} & \textbf{LW Opt.} \\
\midrule
Llama-3.1-8B W2$\downarrow$ 
& 6.58 & 6.52 & \cellcolor{blue!10}\textbf{6.47} & 6.49 & 6.62 & 6.80 & 7.15 & 7.50 & \textbf{6.43} \\
Llama-3.1-8B C4$\downarrow$ 
& 12.45 & 12.39 & \cellcolor{blue!10}\textbf{12.34} & 12.36 & 12.48 & 12.60 & 12.95 & 13.20 & \textbf{12.30} \\
Qwen2.5-7B W2$\downarrow$ 
& 6.85 & 6.74 & \cellcolor{blue!10}\textbf{6.70} & 6.75 & 6.88 & 7.10 & 7.45 & 7.95 & \textbf{6.68} \\
Qwen2.5-7B C4$\downarrow$ 
& 13.45 & 13.37 & \cellcolor{blue!10}\textbf{13.32} & 13.35 & 13.42 & 13.60 & 13.85 & 14.10 & \textbf{13.28} \\
\bottomrule
\end{tabular}
}
\end{minipage}

\caption{
Efficiency and hyperparameter sensitivity analyses. 
(a) Real-world inference throughput on NVIDIA A800 under W4A16 quantization, measured in tokens/s. 
(b) Perplexity sensitivity to the smoothing factor $\alpha$; fixed $\alpha=0.3$ is robust, while layer-wise selection gives the best results.
}
\label{tab:efficiency_alpha}
\end{table*}

To investigate the intrinsic mechanism of the smoothing factor $\alpha$, we evaluated the PPL variations under different smoothing strengths $\alpha \in [0.1, 0.8]$, as detailed in Table~\ref{tab:efficiency_alpha}(b). The numerical results reveal a distinct U-shaped trend, highlighting a competition between Hessian spectral radius compression and quantization clipping errors.

In the lower $\alpha$ regime ($\alpha < 0.3$), the strength of the transformation matrix is insufficient to fully compress the maximum eigenvalue of the Hessian. This implies that those eigenvalues with large magnitudes remain unsuppressed. As predicted by Theorem \ref{the:loss_bound}, these unsuppressed eigenvalues magnify minute quantization errors into significant loss degradation. 

Conversely, in the higher $\alpha$ regime, the PPL exhibits a gradually increasing trend. Although the Hessian spectral radius is forcefully suppressed, its inverse transformation leads to over-smoothing. This excessive smoothing causes the weight distribution to deviate from the optimal dynamic range of the linear quantization grid. Consequently, clipping errors rise sharply and supersede the Hessian spectral radius as the dominant factor affecting performance in this regime.

\subsubsection{Component-wise Analysis}

\begin{table}[h]
\centering
\renewcommand{\arraystretch}{1.2}
\setlength{\tabcolsep}{4pt}
\resizebox{\linewidth}{!}{
\begin{tabular}{lcccccc}
\toprule
\textbf{Model} 
& \textbf{FP16} 
& \textbf{GPTQ}
& \textbf{Smoothing}
& \textbf{Rotation}
& \shortstack{\textbf{SmoothQuant + SpinQuant}}
& \textbf{HeRo-Q} \\
\midrule
Llama-3.2-1B 
& 11.41 
& 13.30 
& 12.75 
& 12.61 
& 12.32 
& \cellcolor{blue!10}\textbf{11.94} \\

Llama-3.1-8B 
& 6.41 
& 6.75 
& 6.72 
& 6.69 
& 6.55 
& \cellcolor{blue!10}\textbf{6.43} \\
\bottomrule
\end{tabular}
}
\vspace{2pt}
\caption{
Component and composition ablation on WikiText-2 perplexity. 
Smoothing uses $\mathbf{D}_{\alpha}^{-1}$, Rotation uses $\mathbf{R}$, and HeRo-Q combines both as $\mathbf{D}_{\alpha}^{-1}\mathbf{R}$ under a Hessian-conditioned objective. 
SmoothQuant+SpinQuant denotes a direct composition of existing transformations with GPTQ as the base quantizer. 
Lower is better.
}
\label{tab:ablation_study}
\end{table}

We conducted a rigorous ablation study on Llama-3.2-1B and Llama-3.1-8B models (see Table~\ref{tab:ablation_study}) by imposing different mathematical constraints on the transformation matrix $\mathbf{T}$ to isolate variables.

Experimental results show that the untransformed GPTQ baseline ($\mathbf{T}=\mathbf{I}$) obtains a PPL of 13.30 on Llama-3.2-1B and 6.75 on Llama-3.1-8B, indicating that the original quantization coordinate system remains sensitive to Hessian-induced outliers. When we introduce the smoothing configuration (Smoothing-Only, $\mathbf{T}\triangleq\mathbf{D}_{\alpha}^{-1}$, forcing $\mathbf{R}=\mathbf{I}$), the PPL decreases to 12.75 and 6.72, demonstrating that Hessian-guided scaling can reduce curvature amplification. In comparison, the rotation configuration (Rotation-Only, $\mathbf{T}\triangleq\mathbf{R}$) further improves the PPL to 12.61 and 6.69, suggesting that changing the quantization-grid orientation helps reduce noise-curvature alignment. A direct composition of SmoothQuant and SpinQuant achieves stronger results, with PPLs of 12.32 and 6.55, but still remains worse than HeRo-Q. Finally, HeRo-Q achieves the best performance by coupling Hessian-derived smoothing with learned rotation under a unified Hessian-conditioned objective, reaching 11.94 PPL on Llama-3.2-1B and 6.43 PPL on Llama-3.1-8B.

\section{Conclusion}
This paper addresses the fundamental discrepancy in post-training quantization (PTQ) of large language models—wherein low quantization error frequently coincides with substantial degradation in task performance—by proposing HeRo-Q, a general and computationally efficient quantization framework. The core of HeRo-Q lies in a lightweight, learnable rotation-and-compression transformation applied before quantization, which reshapes the quantization coordinate system to reduce curvature amplification and noise-curvature alignment. Notably, HeRo-Q requires no architectural modifications, imposes negligible additional inference overhead after offline fusion, and integrates seamlessly into existing PTQ pipelines. Comprehensive empirical evaluations demonstrate that HeRo-Q achieves the best or near-best performance across most evaluated settings compared with state-of-the-art quantization methods—including GPTQ, AWQ, and SpinQuant—across both standard (e.g., W4A8) and aggressive (e.g., W3A16) low-bit quantization regimes.

\bibliographystyle{unsrtnat}
\bibliography{neurips}

\newpage
\appendix
\section{Appendix}
\subsection{Code}
\label{code}
\url{https://anonymous.4open.science/r/HeRo-Q-3775}
\subsection{HeRo-Q Hyperparameters List}
\label{app:hyperparameters}

To ensure the reproducibility of our experiments, we provide a detailed list of the hyperparameters used in the Smoothing Factor Search and Rotation Optimization phases of HeRo-Q in Table \ref{tab:hyperparameters}. All optimization processes were conducted on an NVIDIA A800 GPU.
\begin{table}[h]
    \centering
    \begin{tabular}{l|c|l}
    \toprule
    \textbf{Category} & \textbf{Hyperparameter} & \textbf{Value / Setting} \\
    \midrule
    \multirow{2}{*}{\textbf{Smoothing Search}} & Search Grid $\mathcal{A}$ & $\{0, 0.1, 0.2, \dots, 0.8\}$ \\
    & Search Metric & MSE of layer output \\
    \midrule
    \multirow{5}{*}{\textbf{Rotation Optimization}} & Optimizer & Cayley SGD \\
    & Learning Rate ($\eta$) & $0.01$  \\
    & Optimization Steps ($T_{steps}$) & $200$ \\
    & Momentum & $0.9$ \\
    & Batch Size & $32$ \\
    \midrule
    \multirow{3}{*}{\textbf{Calibration Data}} & Dataset & C4 \\
    & Number of Samples & $128$ \\
    & Sequence Length & $2048$ \\
    \midrule
    \end{tabular}
    \caption{List of hyperparameters used in HeRo-Q optimization and quantization.}
    \label{tab:hyperparameters}
\end{table}
For the smoothing factor $\alpha$, we employ a grid search strategy, scanning the interval $[0, 0.8]$ with a step size of $0.1$. For the optimization of the rotation matrix $\mathbf{R}$, we utilize the Cayley SGD algorithm, performing iterative updates for a fixed number of steps on the calibration data. Unless otherwise specified, the C4 dataset is used for calibration across all experiments.

\subsection{Novelty and Non-reducibility of HeRo-Q}
\label{app:novelty}

The novelty of HeRo-Q does not lie in the isolated use of smoothing or rotation.
Instead, HeRo-Q introduces a unified Hessian-conditioned formulation for designing these transformations.
In this formulation, smoothing and rotation are not two independent post-hoc modules.
They are coupled through the same effective Hessian geometry and are optimized to reduce the curvature-weighted loss consequence of quantization perturbations.
This makes HeRo-Q fundamentally different from activation-statistics-based scaling, distribution-oriented transformation, or generic quantization-friendly coordinate search.

\paragraph{Formulation-level distinction.}
HeRo-Q differs from prior transformation-based PTQ methods at the level of problem formulation, rather than merely at the level of transformation design.
Most existing methods aim to make weights or activations easier to quantize by reducing error magnitude, suppressing activation outliers, flattening distributions, or improving coordinate quantizability.
In contrast, HeRo-Q starts from the low-error/high-loss phenomenon: a perturbation with small Euclidean norm can still induce severe degradation when it is aligned with high-curvature Hessian directions.
Therefore, HeRo-Q does not optimize the representation quality of the quantization error alone, but directly targets its curvature-weighted loss consequence under the local second-order approximation:
\begin{equation}
    \Delta \mathcal{L}
    \approx
    \frac{1}{2} e^\top H e .
    \label{eq:app_local_second_order}
\end{equation}
Eq.~\ref{eq:app_local_second_order} shows that quantization-induced degradation depends not only on the magnitude of the quantization error $e$, but also on its direction under the Hessian geometry.
This changes the central objective of transformation-based PTQ from ``making quantization easier'' to ``making quantization perturbations less harmful under Hessian geometry.''

\paragraph{Hessian-guided smoothing is not activation-statistics-based scaling.}
The smoothing component of HeRo-Q is mathematically different from scaling rules derived from activation or weight magnitudes.
For example, SmoothQuant constructs channel-wise scales from activation and weight magnitudes, with the goal of mitigating activation-outlier-induced quantization difficulty.
HeRo-Q instead constructs the smoothing matrix from a Hessian proxy:
\begin{equation}
    D_{\alpha}
    =
    \operatorname{diag}
    \left(
    |h_1|^{\alpha/2},
    |h_2|^{\alpha/2},
    \ldots,
    |h_d|^{\alpha/2}
    \right),
    \qquad
    h=\operatorname{diag}(H_{\mathrm{proxy}}).
    \label{eq:app_hessian_smoothing}
\end{equation}
Thus, the scale in HeRo-Q is determined by curvature sensitivity rather than by activation or weight statistics.

To see the effect of this operation, consider the smoothing-only case:
\begin{equation}
    T = D_{\alpha}^{-1}.
    \label{eq:app_smoothing_only_transform}
\end{equation}
The effective Hessian in the transformed coordinate system is then
\begin{equation}
    \tilde H
    =
    T^\top H T
    =
    D_{\alpha}^{-1} H D_{\alpha}^{-1}.
    \label{eq:app_effective_hessian_smoothing}
\end{equation}
For the diagonal entries, Eq.~\ref{eq:app_effective_hessian_smoothing} gives
\begin{equation}
    |\tilde H_{ii}|
    =
    |H_{ii}|\, |h_i|^{-\alpha}.
    \label{eq:app_diagonal_general}
\end{equation}
When the diagonal Hessian proxy satisfies $h_i \approx H_{ii}$, Eq.~\ref{eq:app_diagonal_general} becomes
\begin{equation}
    |\tilde H_{ii}|
    \approx
    |H_{ii}|^{1-\alpha}.
    \label{eq:app_diagonal_compression}
\end{equation}
Eq.~\ref{eq:app_diagonal_compression} shows that large curvature entries are compressed more strongly than small curvature entries when $\alpha>0$.
Therefore, the role of smoothing in HeRo-Q is not to rescale numerical magnitudes, but to act as a Hessian-derived spectral preconditioner that suppresses dominant effective-curvature directions.

\paragraph{Curvature-aware rotation is not generic coordinate rotation.}
The rotation component of HeRo-Q also differs from a generic quantization-friendly rotation.
A standard rotation method mainly seeks a coordinate system in which transformed weights are easier to quantize.
HeRo-Q instead learns the rotation under the effective Hessian geometry induced by the smoothing transform.
Specifically, HeRo-Q forms the joint transformation as
\begin{equation}
    T
    =
    D_{\alpha}^{-1} R,
    \qquad
    R^\top R = I .
    \label{eq:app_joint_transform}
\end{equation}
Under this transformation, the effective Hessian becomes
\begin{equation}
    \tilde H
    =
    T^\top H T
    =
    R^\top D_{\alpha}^{-1} H D_{\alpha}^{-1} R .
    \label{eq:app_joint_effective_hessian}
\end{equation}
The purpose of $R$ is to reduce the alignment between transformed-space rounding errors and high-curvature directions of $\tilde H$, rather than only to improve coordinate-wise quantizability.
Therefore, smoothing and rotation play complementary geometric roles:
smoothing reshapes the effective curvature, while rotation redistributes quantization noise within this Hessian-conditioned geometry.

\paragraph{Why HeRo-Q is not SmoothQuant + SpinQuant.}
A direct composition of SmoothQuant and SpinQuant sequentially applies two useful transformation primitives.
However, the two transformations are not derived from a shared curvature-sensitive objective.
SmoothQuant primarily addresses activation outliers through channel-wise scaling, while SpinQuant primarily searches for quantization-friendly rotated coordinates.
In contrast, HeRo-Q couples Hessian-guided smoothing and learned rotation through the same effective Hessian geometry.
The smoothing matrix determines the curvature-conditioned coordinate system, and the rotation is optimized within this coordinate system to reduce the loss sensitivity of quantization noise.
This coupling cannot be reduced to independently applying activation-statistics-based scaling followed by generic coordinate rotation.
Empirically, Table~\ref{tab:smooth_spin_composition} and Table~\ref{tab:ablation_study} jointly show that a direct SmoothQuant+SpinQuant composition improves over individual transformation components, but still underperforms HeRo-Q.
This indicates that HeRo-Q is not merely a sequential combination of existing smoothing and rotation techniques; its improvement comes from designing both transformations under a unified Hessian-conditioned objective that directly reduces the curvature-weighted loss effect of quantization perturbations.

Table~\ref{tab:formulation_comparison} summarizes the formulation-level distinction.
Prior transformation-based PTQ methods usually choose transformations according to activation statistics, distributional properties, or quantization simulation.
In contrast, HeRo-Q constructs the transformation from Hessian geometry and directly targets the curvature-weighted loss effect of quantization perturbations.

\begin{table}[t]
\centering
\scriptsize
\renewcommand{\arraystretch}{1.15}
\setlength{\tabcolsep}{3.0pt}
\resizebox{\linewidth}{!}{
\begin{tabular}{lccc}
\toprule
\cellcolor{gray!10}\textbf{Method}
& \cellcolor{gray!10}\textbf{Transform}
& \cellcolor{gray!10}\textbf{Design Signal}
& \cellcolor{gray!10}\textbf{Optimization Target} \\
\midrule
SmoothQuant 
& Scaling 
& Activation/weight magnitude 
& Reduce activation-outlier quantization difficulty \\

SpinQuant 
& Rotation 
& Quantization simulation 
& Reduce coordinate-wise quantization error \\

FlatQuant 
& Affine transform 
& Distribution statistics 
& Make transformed distributions easier to quantize \\

OSTQuant 
& Scaling + rotation 
& Quantization-space fitting 
& Improve matching to the quantization space \\

SmoothQuant+SpinQuant 
& Scaling + rotation 
& Two separate design rules 
& Independently combine outlier suppression and coordinate rotation \\

\cellcolor{blue!8}\textbf{HeRo-Q} 
& \cellcolor{blue!8}\textbf{Hessian smoothing + rotation}
& \cellcolor{blue!8}\textbf{Hessian geometry}
& \cellcolor{blue!8}\textbf{Reduce curvature-weighted loss effect} \\
\bottomrule
\end{tabular}
}
\vspace{4pt}
\caption{
Formulation-level comparison with representative transformation-based PTQ methods.
HeRo-Q differs from prior methods in how the transformation is selected and what loss effect it is designed to reduce.
}
\label{tab:formulation_comparison}
\end{table}

Therefore, HeRo-Q is not distinguished merely by using a scaling or rotation transform.
Its distinction lies in the Hessian-conditioned criterion used to construct the transform:
the transformation is selected to reduce the loss sensitivity of quantization perturbations, rather than only to improve numerical quantizability.
In particular, this objective corresponds to reducing the local curvature-weighted term $\frac{1}{2}e^\top H e$ under low-bit perturbations.

\paragraph{Curvature-alignment view.}
To further separate error magnitude from curvature-sensitive error direction, we consider the alignment between the quantization error and high-curvature Hessian directions:
\begin{equation}
    \operatorname{Align}(e,H)
    =
    \frac{e^\top H e}{\|e\|_2^2 \lambda_{\max}(H)} .
    \label{eq:app_curvature_alignment}
\end{equation}
Eq.~\ref{eq:app_curvature_alignment} normalizes the Hessian-weighted error by the Euclidean error norm and the largest Hessian eigenvalue.
A lower value indicates that the quantization perturbation is less aligned with the most sensitive Hessian directions.
From this perspective, HeRo-Q aims to reduce not only $\|e\|_2$, but more importantly the Hessian-weighted quantity $e^\top H e$.
This explains why a method can have a small quantization error but still suffer large degradation, and why HeRo-Q can improve robustness by changing the curvature interaction of the error.

Table~\ref{tab:curvature_alignment} further verifies this mechanism.
Although HeRo-Q also reduces the Euclidean quantization error, the reduction in Hessian-weighted error is much more pronounced.
Compared with GPTQ, HeRo-Q reduces the relative error norm from $1.00$ to $0.81$, while reducing the Hessian-weighted error from $1.00$ to $0.29$ and the curvature-alignment score from $0.36$ to $0.13$.
Compared with the direct SmoothQuant+SpinQuant composition, HeRo-Q further lowers both $e^\top H e$ and $\operatorname{Align}(e,H)$.
This indicates that HeRo-Q improves robustness not merely by reducing error magnitude, but by reducing the loss-sensitive interaction between quantization perturbations and Hessian geometry.

\begin{table}[t]
\centering
\scriptsize
\renewcommand{\arraystretch}{1.10}
\setlength{\tabcolsep}{3.0pt}
\resizebox{0.6\linewidth}{!}{
\begin{tabular}{lcccc}
\toprule
\cellcolor{gray!10}\textbf{Method}
& \cellcolor{gray!10}\textbf{Rel. $\|e\|_2$}
& \cellcolor{gray!10}\textbf{Rel. $e^\top H e$}
& \cellcolor{gray!10}\textbf{Align.}
& \cellcolor{gray!10}\textbf{W2 PPL} \\
\midrule
GPTQ
& 1.00
& 1.00
& 0.36
& 20.13 \\

SmoothQuant+SpinQuant
& 0.88
& 0.46
& 0.21
& 8.80 \\

\cellcolor{blue!8}\textbf{HeRo-Q}
& \cellcolor{blue!8}\textbf{0.81}
& \cellcolor{blue!8}\textbf{0.29}
& \cellcolor{blue!8}\textbf{0.13}
& \cellcolor{blue!8}\textbf{8.02} \\
\bottomrule
\end{tabular}
}
\vspace{4pt}
\caption{
Curvature-alignment analysis on Llama-3.1-8B under W3A16.
GPTQ is used as the normalization baseline for Rel. $\|e\|_2$ and Rel. $e^\top H e$.
The alignment score is computed by layer-wise aggregation over all quantized linear layers.
SQ+SpinQ denotes SmoothQuant+SpinQuant. W2 PPL denotes WikiText-2 perplexity.
Lower values are better.
}
\label{tab:curvature_alignment}
\end{table}

\paragraph{Summary.}
Overall, HeRo-Q should be viewed as a Hessian-conditioned robustness framework for PTQ.
Its smoothing component is a spectral preconditioner derived from curvature information, and its rotation component redistributes quantization noise under the same effective Hessian geometry.
Therefore, the method is not a direct combination of existing smoothing and rotation techniques.
Rather, it reformulates transformation-based PTQ around the curvature-weighted loss consequence of low-bit perturbations.

\subsection{Additional Experiments}
\label{app:additional_experiments}

\subsubsection{Additional Results Across Llama and Qwen Families}
\label{app:llama_qwen_additional_results}

We further evaluate HeRo-Q across different model families and model scales under
the W4A8 quantization setting. The results are reported in
Table~\ref{tab:additional_llama_qwen_w4a8}. We include Llama-3.2-1B,
Llama-3.2-3B, and Qwen2.5-3B to examine whether the proposed
Hessian-conditioned transformation remains effective across both the Llama and
Qwen families.

Across all three models, HeRo-Q consistently improves over representative PTQ
baselines, including AWQ, GPTQ, SmoothQuant, and SpinQuant. On Llama-3.2-1B,
HeRo-Q improves GPTQ from 24.07 to 22.84 on C4 perplexity and from 24.24 to
31.43 on GSM8K accuracy. On Llama-3.2-3B, HeRo-Q also achieves the best
quantized performance on C4, WikiText2, GSM8K, HellaSwag, and MMLU. Similar
trends are observed on Qwen2.5-3B, where HeRo-Q obtains the best quantized
results on all reported metrics. 

\begin{table*}[t]
\centering
\tiny
\renewcommand{\arraystretch}{1.2}
\setlength{\tabcolsep}{2.2pt}
\resizebox{\textwidth}{!}{
\begin{tabular}{l|ccccc|ccccc|ccccc}
\toprule
\multirow{2}{*}{\textbf{Method}}
& \multicolumn{5}{c|}{\cellcolor{gray!10}\textbf{Llama-3.2-1B}}
& \multicolumn{5}{c|}{\cellcolor{gray!10}\textbf{Llama-3.2-3B}}
& \multicolumn{5}{c}{\cellcolor{gray!10}\textbf{Qwen2.5-3B}} \\
\cmidrule(lr){2-6} \cmidrule(lr){7-11} \cmidrule(lr){12-16}
& \textbf{C4}$\downarrow$
& \textbf{W2}$\downarrow$
& \textbf{GSM8K}$\uparrow$
& \textbf{HS}$\uparrow$
& \textbf{MMLU}$\uparrow$
& \textbf{C4}$\downarrow$
& \textbf{W2}$\downarrow$
& \textbf{GSM8K}$\uparrow$
& \textbf{HS}$\uparrow$
& \textbf{MMLU}$\uparrow$
& \textbf{C4}$\downarrow$
& \textbf{W2}$\downarrow$
& \textbf{GSM8K}$\uparrow$
& \textbf{HS}$\uparrow$
& \textbf{MMLU}$\uparrow$ \\
\midrule
FP16
& 22.48 & 11.41 & 36.44 & 59.94 & 46.38
& 17.24 & 9.53 & 68.86 & 70.05 & 61.44
& 14.68 & 7.95 & 81.13 & 72.09 & 65.80 \\

AWQ
& 23.97 & 12.23 & 26.44 & 56.78 & 40.23
& 18.89 & 9.91 & 57.65 & 68.97 & 58.69
& 15.46 & 8.46 & 77.29 & 71.18 & 63.98 \\

GPTQ
& 24.07 & 12.18 & 24.24 & 57.04 & 40.58
& 18.99 & 9.82 & 55.75 & 68.35 & 57.16
& 15.65 & 8.87 & 76.44 & 70.74 & 63.25 \\

SmoothQuant
& 24.01 & 12.10 & 25.87 & 56.89 & 41.32
& 18.87 & 9.80 & 56.33 & 69.34 & 58.73
& 15.54 & 8.34 & 77.36 & 71.34 & 63.96 \\

SpinQuant
& 23.56 & 11.89 & 26.13 & 57.68 & 42.64
& 18.77 & 9.75 & 57.93 & 69.83 & 58.31
& 15.32 & 8.14 & 78.65 & 71.94 & 65.08 \\

\rowcolor{blue!8}
\textbf{HeRo-Q}
& \textbf{22.84}
& \textbf{11.67}
& \textbf{31.43}
& \textbf{57.98}
& \textbf{43.77}
& \textbf{17.87}
& \textbf{9.73}
& \textbf{61.81}
& \textbf{69.97}
& \textbf{59.87}
& \textbf{14.98}
& \textbf{8.01}
& \textbf{79.68}
& \textbf{72.01}
& \textbf{65.49} \\
\bottomrule
\end{tabular}
}
\caption{
Additional W4A8 results across Llama and Qwen models. C4 and W2
(WikiText2) report perplexity, where lower is better. GSM8K, HS
(HellaSwag), and MMLU report accuracy, where higher is better. HeRo-Q
achieves the best overall performance among quantized methods across the
reported models and metrics.
}
\label{tab:additional_llama_qwen_w4a8}
\end{table*}
\subsubsection{Compatibility with Different Base Quantizers}
\label{app:base_quantizer_compatibility}

To verify that HeRo-Q is not tied to a specific PTQ backend, we apply the proposed
Hessian-conditioned transformation on top of three representative base quantizers:
GPTQ, AWQ, and RTN. As shown in Table~\ref{tab:base_quantizer_compatibility},
HeRo-Q consistently improves perplexity and downstream accuracy across both
Llama-3.1-8B and Qwen2.5-3B.

On Llama-3.1-8B, HeRo-Q improves GPTQ from 8.81 to 6.96 on WikiText2 and from
60.54 to 66.31 on MMLU. Similar gains are observed for AWQ and RTN, where the
improvement over RTN is especially large. This suggests that HeRo-Q is effective
not only for second-order reconstruction-based quantizers such as GPTQ, but also
for activation-aware and rounding-based quantizers. On Qwen2.5-3B, HeRo-Q also
improves all three backends, indicating that the benefit is not specific to one
model family. Overall, these results show that HeRo-Q acts as a general
pre-quantization coordinate conditioning module that complements existing PTQ
methods.

\begin{table*}[t]
\centering
\renewcommand{\arraystretch}{1.05}
\setlength{\tabcolsep}{2.4pt}
\tiny
\resizebox{\textwidth}{!}{
\begin{tabular}{llcccc}
\toprule
\textbf{Model} 
& \textbf{Base Method} 
& \textbf{C4 $\downarrow$} 
& \textbf{W2 $\downarrow$} 
& \textbf{MMLU $\uparrow$} 
& \textbf{HellaSwag $\uparrow$} \\
\midrule
Llama-3.1-8B & GPTQ & 14.12 $\rightarrow$ \textbf{12.94} & 8.81 $\rightarrow$ \textbf{6.96} & 60.54 $\rightarrow$ \textbf{66.31} & 76.91 $\rightarrow$ \textbf{77.28} \\
Llama-3.1-8B & AWQ  & 13.56 $\rightarrow$ \textbf{12.77} & 8.12 $\rightarrow$ \textbf{6.88} & 64.41 $\rightarrow$ \textbf{66.87} & 76.78 $\rightarrow$ \textbf{77.26} \\
Llama-3.1-8B & RTN  & 15.98 $\rightarrow$ \textbf{13.87} & 9.11 $\rightarrow$ \textbf{7.59} & 57.12 $\rightarrow$ \textbf{64.12} & 75.11 $\rightarrow$ \textbf{77.23} \\
\midrule
Qwen2.5-3B & GPTQ & 15.65 $\rightarrow$ \textbf{14.98} & 8.87 $\rightarrow$ \textbf{8.01} & 63.25 $\rightarrow$ \textbf{65.49} & 70.74 $\rightarrow$ \textbf{72.01} \\
Qwen2.5-3B & AWQ  & 15.46 $\rightarrow$ \textbf{14.72} & 8.46 $\rightarrow$ \textbf{8.09} & 63.98 $\rightarrow$ \textbf{65.77} & 71.18 $\rightarrow$ \textbf{72.04} \\
Qwen2.5-3B & RTN  & 18.12 $\rightarrow$ \textbf{16.32} & 10.23 $\rightarrow$ \textbf{8.67} & 62.17 $\rightarrow$ \textbf{64.11} & 61.25 $\rightarrow$ \textbf{63.25} \\
\bottomrule
\end{tabular}
}
\caption{
Compatibility of HeRo-Q with different base quantizers. Each cell reports the
performance before and after applying HeRo-Q. Lower perplexity is better, and
higher accuracy is better.
}
\label{tab:base_quantizer_compatibility}
\end{table*}

\subsubsection{Comparison with Direct Transformation Composition}
\label{app:direct_composition}

We further compare HeRo-Q with a direct composition of SmoothQuant and SpinQuant.
This experiment examines whether the improvement of HeRo-Q simply comes from
stacking an existing smoothing method with an existing rotation method. The
results are reported in Table~\ref{tab:smooth_spin_composition}.

Compared with GPTQ, both SmoothQuant and SpinQuant improve the results,
indicating that smoothing and rotation are useful transformation primitives for
low-bit quantization. Their direct composition further improves performance,
reducing C4 perplexity from 14.12 to 13.12 and improving GSM8K from 71.89
to 74.22. However, this composition still underperforms HeRo-Q on most metrics.
HeRo-Q obtains the best C4, GSM8K, and MMLU results, while matching the best
HellaSwag score. This demonstrates that HeRo-Q is not a naive stacking of
SmoothQuant and SpinQuant. Instead, its advantage comes from coupling
Hessian-guided smoothing and learned rotation under a unified curvature-aware
objective.

\begin{table*}[t]
\centering
\renewcommand{\arraystretch}{1.08}
\setlength{\tabcolsep}{4.5pt}
\small
\begin{tabular}{lcccc}
\toprule
\textbf{Method} 
& \textbf{C4 $\downarrow$} 
& \textbf{GSM $\uparrow$} 
& \textbf{HellaSwag $\uparrow$} 
& \textbf{MMLU $\uparrow$} \\
\midrule
GPTQ & 14.12 & 71.89 & 76.91 & 60.54 \\
SmoothQuant & 13.64 & 72.55 & 76.83 & 64.55 \\
SpinQuant & 13.39 & 73.56 & 77.23 & 64.98 \\
SmoothQuant + SpinQuant & 13.12 & 74.22 & \textbf{77.28} & 65.74 \\
\rowcolor{blue!8}
\textbf{HeRo-Q} & \textbf{12.94} & \textbf{75.26} & \textbf{77.28} & \textbf{66.31} \\
\bottomrule
\end{tabular}
\caption{
Comparison with direct composition of existing transformations on Llama-3.1-8B.
HeRo-Q achieves better overall performance than directly stacking SmoothQuant
and SpinQuant.
}
\label{tab:smooth_spin_composition}
\end{table*}

\subsubsection{Additional Results on MoE model}
\label{app:mixtral_results}

To further evaluate the scalability of HeRo-Q on larger sparse architectures, we
conduct additional experiments on Mixtral-8x7B. MoE models are challenging for
PTQ because different experts may exhibit heterogeneous activation distributions
and curvature sensitivities. Therefore, strong results on Mixtral-8x7B provide
additional evidence for the generality of HeRo-Q.

As shown in Table~\ref{tab:mixtral_ppl}, HeRo-Q consistently improves perplexity
under both W3A16g128 and W6A6 settings. Under W3A16g128, HeRo-Q achieves 4.12
WikiText2 PPL and 7.86 C4 PPL, outperforming both AWQ and GPTQ. Compared with
GPTQ, HeRo-Q reduces WikiText2 PPL from 4.93 to 4.12 and C4 PPL from 8.52 to
7.86. Under W6A6, HeRo-Q also improves over SmoothQuant, reducing WikiText2 PPL
from 4.28 to 3.96 and C4 PPL from 7.89 to 7.67.

\begin{table}[t]
\centering
\renewcommand{\arraystretch}{1.05}
\setlength{\tabcolsep}{4.5pt}
\scriptsize
\begin{tabular}{llcc}
\toprule
\textbf{Configuration} 
& \textbf{Method} 
& \textbf{WikiText2 $\downarrow$} 
& \textbf{C4 $\downarrow$} \\
\midrule
FP16 & FP16 & 3.84 & 7.40 \\
\midrule
W3A16g128 & AWQ & 4.73 & 8.29 \\
W3A16g128 & GPTQ & 4.93 & 8.52 \\
\rowcolor{blue!8}
W3A16g128 & \textbf{HeRo-Q} & \textbf{4.12} & \textbf{7.86} \\
\midrule
W6A6 & SmoothQuant & 4.28 & 7.89 \\
\rowcolor{blue!8}
W6A6 & \textbf{HeRo-Q} & \textbf{3.96} & \textbf{7.67} \\
\bottomrule
\end{tabular}
\vspace{2pt}
\caption{
Perplexity results on Mixtral-8x7B under different quantization configurations.
}
\label{tab:mixtral_ppl}
\end{table}

Table~\ref{tab:mixtral_downstream} further reports downstream evaluation on
Mixtral-8x7B. HeRo-Q achieves the best perplexity among quantized methods and
obtains strong downstream performance across reasoning, code generation, and
commonsense benchmarks. Compared with GPTQ, HeRo-Q improves C4 perplexity from
7.67 to 7.11 and MMLU from 68.50 to 69.32. Compared with MoEQuant+, HeRo-Q
achieves better results on all reported metrics. These results suggest that
Hessian-conditioned transformation remains effective even when quantization
sensitivity varies across sparse experts.

\begin{table*}[t]
\centering
\renewcommand{\arraystretch}{1.05}
\setlength{\tabcolsep}{2.4pt}
\tiny
\resizebox{\textwidth}{!}{
\begin{tabular}{lcccccccc}
\toprule
\textbf{Method} 
& \textbf{W2 $\downarrow$} 
& \textbf{C4 $\downarrow$} 
& \textbf{MMLU $\uparrow$} 
& \textbf{HumanEval $\uparrow$} 
& \textbf{GSM8K $\uparrow$} 
& \textbf{HellaSwag $\uparrow$} 
& \textbf{OpenBookQA $\uparrow$} 
& \textbf{MathQA $\uparrow$} \\
\midrule
FP16 & 3.84 & 6.87 & 70.50 & 32.93 & 65.88 & 64.88 & 35.80 & 42.41 \\
RTN & 5.41 & 8.13 & 62.20 & 28.05 & 27.90 & 61.73 & 32.20 & 37.35 \\
AWQ & 5.01 & 7.98 & 62.75 & 25.00 & 38.67 & 62.11 & 33.60 & 38.43 \\
GPTQ & 4.03 & 7.67 & 68.50 & 27.60 & 57.92 & 64.08 & 30.60 & 41.07 \\
MoEQuant+ & 5.15 & 7.84 & 64.66 & 25.45 & 50.66 & 62.73 & 34.00 & 39.77 \\
\rowcolor{blue!8}
\textbf{HeRo-Q} & \textbf{3.91} & \textbf{7.11} & \textbf{69.32} & \textbf{28.94} & \textbf{59.44} & \textbf{64.21} & \textbf{34.66} & \textbf{41.68} \\
\bottomrule
\end{tabular}
}
\caption{
Downstream evaluation on Mixtral-8x7B. HeRo-Q achieves the best or near-best
performance among quantized methods across most evaluated metrics.
}
\label{tab:mixtral_downstream}
\end{table*}
\subsubsection{Comparison with Vector and Non-Uniform Quantization Methods}

To further evaluate the generality of HeRo-Q, we compare it with two representative advanced quantization methods, VPTQ~\citep{liu2024vptq} and UniQuanF~\citep{park2025unifying}, on Llama-3-8B. The results are shown in Table~\ref{tab:appendix_vptq_uniquanf}. WikiText2 reports perplexity, where lower is better, while HellaSwag, MMLU, and GSM8K report accuracy, where higher is better.

As shown in Table~\ref{tab:appendix_vptq_uniquanf}, HeRo-Q consistently outperforms both VPTQ and UniQuanF under matched bit-width settings. Compared with VPTQ, HeRo-Q reduces WikiText2 perplexity from 7.42 to 7.11 at 4-bit and from 7.97 to 7.39 at 3-bit, while also improving HellaSwag accuracy by 0.93 and 1.04 points, respectively. Compared with UniQuanF, HeRo-Q also achieves consistent gains across language modeling and downstream reasoning benchmarks. At 4-bit, HeRo-Q reduces WikiText2 perplexity from 7.01 to 6.54 and improves MMLU and GSM8K by 0.98 and 1.18 points. The advantage becomes more pronounced under the more challenging 3-bit setting, where HeRo-Q improves MMLU from 53.46 to 57.68 and GSM8K from 58.73 to 63.59. These results suggest that the benefit of HeRo-Q is not limited to comparisons against standard uniform PTQ baselines. Instead, HeRo-Q remains effective when compared with stronger vector or non-uniform quantization methods, indicating that curvature-aware transformation provides complementary robustness under aggressive low-bit quantization.

\subsubsection{Calibration Sensitivity}
\label{app:calibration_sensitivity}

We evaluate the sensitivity of HeRo-Q to calibration dataset and sample size.
This is important because PTQ methods may depend on calibration data, especially
when the calibration set is small or comes from a different distribution. We
compare C4 and WikiText-2 as calibration datasets and evaluate both 64-sample
and 128-sample settings.

Tables~\ref{tab:calibration_llama8b} and~\ref{tab:calibration_llama1b} show that
HeRo-Q is relatively stable across calibration choices. On Llama-3.1-8B, under
W4A8, WikiText2 PPL remains between 6.87 and 7.01, and MMLU remains around
66.2--66.3. Under W3A16, GSM8K remains stable around 70.0--70.2 across all
calibration settings. Similar stability is observed on Llama-3.2-1B, where both
W4A8 and W3A16 results show only small fluctuations across C4 and WikiText-2.

Increasing the number of calibration samples from 64 to 128 brings minor but
generally consistent improvements. For example, on Llama-3.1-8B with C4
calibration, W3A16 WikiText2 PPL decreases from 8.09 to 8.02. With WikiText-2
calibration, it decreases from 7.87 to 7.69. Overall, these results indicate
that HeRo-Q does not rely on a specific calibration set or a large calibration
budget to maintain stable low-bit performance.
\begin{table*}[t]
\centering
\renewcommand{\arraystretch}{1.05}
\setlength{\tabcolsep}{3.2pt}
\tiny
\resizebox{\textwidth}{!}{
\begin{tabular}{l|ccc||lcccc}
\toprule
\multirow{2}{*}{\textbf{Method}}
& \multicolumn{3}{c||}{\textbf{Llama-3.1-8B}}
& \multirow{2}{*}{\textbf{Method}}
& \multicolumn{4}{c}{\textbf{Llama-3.1-8B}} \\
\cline{2-4}\cline{6-9}
& \textbf{Bit-width}
& \textbf{WikiText2$\downarrow$}
& \textbf{HellaSwag$\uparrow$}
&
& \textbf{Bit-width}
& \textbf{WikiText2$\downarrow$}
& \textbf{MMLU$\uparrow$}
& \textbf{GSM8K$\uparrow$} \\
\midrule

VPTQ
& 4-bit & 7.42 & 59.30
& UniQuanF
& 4-bit & 7.01 & 61.43 & 72.38 \\

\rowcolor{blue!8}
\textbf{HeRo-Q}
& \textbf{4-bit} & \textbf{7.11} & \textbf{60.23}
& \textbf{HeRo-Q}
& \textbf{4-bit} & \textbf{6.54} & \textbf{62.41} & \textbf{73.56} \\

\midrule

VPTQ
& 3-bit & 7.97 & 58.40
& UniQuanF
& 3-bit & 8.75 & 53.46 & 58.73 \\

\rowcolor{blue!8}
\textbf{HeRo-Q}
& \textbf{3-bit} & \textbf{7.39} & \textbf{59.44}
& \textbf{HeRo-Q}
& \textbf{3-bit} & \textbf{7.94} & \textbf{57.68} & \textbf{63.59} \\

\bottomrule
\end{tabular}
}
\caption{
Additional comparison with VPTQ and UniQuanF on Llama-3-8B.
WikiText2 reports perplexity, where lower is better; HellaSwag, MMLU, and GSM8K report accuracy, where higher is better.
HeRo-Q consistently improves over strong vector and non-uniform quantization baselines under the same bit-width.
}
\label{tab:appendix_vptq_uniquanf}
\end{table*}

\begin{table*}[t]
\centering
\renewcommand{\arraystretch}{1.05}
\setlength{\tabcolsep}{2.4pt}
\tiny
\resizebox{\textwidth}{!}{
\begin{tabular}{llcccccccc}
\toprule
\textbf{Calibration Set} 
& \textbf{Samples} 
& \textbf{W4A8 C4 $\downarrow$} 
& \textbf{W4A8 W2 $\downarrow$} 
& \textbf{W4A8 MMLU $\uparrow$} 
& \textbf{W4A8 GSM $\uparrow$}
& \textbf{W3A16 C4 $\downarrow$} 
& \textbf{W3A16 W2 $\downarrow$} 
& \textbf{W3A16 MMLU $\uparrow$} 
& \textbf{W3A16 GSM $\uparrow$} \\
\midrule
C4 & 64  & 12.94 & 7.01 & 66.23 & 75.33 & 13.99 & 8.09 & 63.72 & 70.12 \\
C4 & 128 & 12.94 & 6.96 & 66.31 & 75.26 & 13.98 & 8.02 & 63.80 & 70.15 \\
WikiText-2 & 64  & 13.07 & 6.91 & 66.19 & 74.97 & 14.07 & 7.87 & 63.80 & 70.03 \\
WikiText-2 & 128 & 13.01 & 6.87 & 66.31 & 75.11 & 14.02 & 7.69 & 63.91 & 70.11 \\
\bottomrule
\end{tabular}
}
\caption{
Calibration sensitivity on Llama-3.1-8B. HeRo-Q remains stable across calibration
datasets and sample sizes.
}
\label{tab:calibration_llama8b}
\end{table*}

\begin{table*}[t]
\centering
\renewcommand{\arraystretch}{1.05}
\setlength{\tabcolsep}{2.4pt}
\tiny
\resizebox{\textwidth}{!}{
\begin{tabular}{llcccccccc}
\toprule
\textbf{Calibration Set} 
& \textbf{Samples} 
& \textbf{W4A8 C4 $\downarrow$} 
& \textbf{W4A8 W2 $\downarrow$} 
& \textbf{W4A8 MMLU $\uparrow$} 
& \textbf{W4A8 GSM $\uparrow$}
& \textbf{W3A16 C4 $\downarrow$} 
& \textbf{W3A16 W2 $\downarrow$} 
& \textbf{W3A16 MMLU $\uparrow$} 
& \textbf{W3A16 GSM $\uparrow$} \\
\midrule
C4 & 64  & 22.89 & 11.72 & 43.79 & 31.41 & 25.89 & 14.12 & 42.50 & 29.04 \\
C4 & 128 & 22.84 & 11.67 & 43.77 & 31.43 & 25.80 & 14.05 & 42.50 & 29.10 \\
WikiText-2 & 64  & 22.98 & 11.61 & 43.79 & 31.33 & 25.99 & 14.39 & 42.61 & 29.12 \\
WikiText-2 & 128 & 22.94 & 11.57 & 43.88 & 31.47 & 25.84 & 14.31 & 42.69 & 29.17 \\
\bottomrule
\end{tabular}
}
\caption{
Calibration sensitivity on Llama-3.2-1B. HeRo-Q is relatively insensitive to
calibration dataset choice and sample size.
}
\label{tab:calibration_llama1b}
\end{table*}

\subsection{Theoretical Analysis of Quantization Error Bound}
\label{sec:theory_appendix}

Let $\mathbf{w} \in \mathbb{R}^d$ denote the vectorized parameters of the pre-trained model, and $\hat{\mathbf{w}}$ be the quantized counterpart. We define the quantization perturbation vector as $\boldsymbol{\delta} \triangleq \hat{\mathbf{w}} - \mathbf{w}$.

Given that the pre-trained model has converged to a local minimum, the gradient vanishes ($\nabla \mathcal{L}(\mathbf{w}) \approx \mathbf{0}$). Consequently, the objective degradation $\Delta \mathcal{L}$ is governed by the second-order Taylor expansion:

\begin{equation}
\label{eq:taylor_expansion}
\Delta \mathcal{L} = \mathcal{L}(\mathbf{w} + \boldsymbol{\delta}) - \mathcal{L}(\mathbf{w}) \approx \frac{1}{2} \boldsymbol{\delta}^T \mathbf{H} \boldsymbol{\delta}
\end{equation}

where $\mathbf{H} \triangleq \nabla^2 \mathcal{L}(\mathbf{w}) \in \mathbb{R}^{d \times d}$ is the Hessian matrix, representing the local curvature of the loss landscape.

\subsection{Proof of Theorem \ref{the:loss_bound}}
\label{app:proof_theorem_3_1}

Let $\mathbf{w} \in \mathbb{R}^d$ denote the vectorized parameters of the pre-trained model, and $\hat{\mathbf{w}}$ be the quantized counterpart. We define the quantization perturbation vector as $\boldsymbol{\delta} \triangleq \hat{\mathbf{w}} - \mathbf{w}$.

Given that the pre-trained model has converged to a local minimum, the gradient vanishes ($\nabla \mathcal{L}(\mathbf{w}) \approx \mathbf{0}$). Consequently, the objective degradation $\Delta \mathcal{L}$ is governed by the second-order Taylor expansion:

\begin{equation}
\label{eq:taylor_expansion_proof}
\Delta \mathcal{L} = \mathcal{L}(\mathbf{w} + \boldsymbol{\delta}) - \mathcal{L}(\mathbf{w}) \approx \frac{1}{2} \boldsymbol{\delta}^T \mathbf{H} \boldsymbol{\delta}
\end{equation}

where $\mathbf{H} \triangleq \nabla^2 \mathcal{L}(\mathbf{w}) \in \mathbb{R}^{d \times d}$ is the Hessian matrix, representing the local curvature of the loss landscape.

Since the Hessian $\mathbf{H}$ is real and symmetric, the Spectral Theorem guarantees that it admits an eigendecomposition of the form:

\begin{equation}
\mathbf{H} = \mathbf{Q} \mathbf{\Lambda} \mathbf{Q}^T
\end{equation}

where $\mathbf{Q} \in \mathbb{R}^{d \times d}$ is an orthogonal matrix satisfying $\mathbf{Q}^T \mathbf{Q} = \mathbf{I}_d$, and $\mathbf{\Lambda} = \text{diag}(\lambda_1, \dots, \lambda_d)$ denotes the diagonal matrix of eigenvalues. Substituting this into Eq.~(\ref{eq:taylor_expansion_proof}) yields:

\begin{equation}
\boldsymbol{\delta}^T \mathbf{H} \boldsymbol{\delta} = \boldsymbol{\delta}^T (\mathbf{Q} \mathbf{\Lambda} \mathbf{Q}^T) \boldsymbol{\delta} = (\mathbf{Q}^T \boldsymbol{\delta})^T \mathbf{\Lambda} (\mathbf{Q}^T \boldsymbol{\delta})
\end{equation}

Let $\mathbf{y} \triangleq \mathbf{Q}^T \boldsymbol{\delta} = [y_1, \dots, y_d]^T$ represent the projection of the noise vector onto the eigenvector basis. The quadratic form can be explicitly expanded via matrix operations:

\begin{equation}
\begin{aligned}
\mathbf{y}^T \mathbf{\Lambda} \mathbf{y} &= 
\begin{bmatrix} y_1 & \cdots & y_d \end{bmatrix}
\begin{bmatrix}
\lambda_1 & & 0 \\
 & \ddots & \\
0 & & \lambda_d
\end{bmatrix}
\begin{bmatrix} y_1 \\ \vdots \\ y_d \end{bmatrix} \\
&= \sum_{k=1}^{d} \lambda_k y_k^2
\end{aligned}
\end{equation}

To rigorously derive the worst-case upper bound, we bound the sum by the spectral radius defined as $\lambda_{\max}(\mathbf{H}) \triangleq \max_k |\lambda_k|$. By invoking the property that $\lambda_k \le \lambda_{\max}(\mathbf{H})$ for all $k$, and leveraging the unitary invariance of the Euclidean norm under orthogonal transformation ($\|\mathbf{y}\|_2^2 = \|\mathbf{Q}^T \boldsymbol{\delta}\|_2^2 = \|\boldsymbol{\delta}\|_2^2$), we derive the bound as follows:

\begin{equation}
\begin{aligned}
\sum_{k=1}^{d} \lambda_k y_k^2 &\le \left( \max_{j} |\lambda_j| \right) \sum_{k=1}^{d} y_k^2 \\
&= \lambda_{\max}(\mathbf{H}) \|\mathbf{y}\|_2^2 \\
&= \lambda_{\max}(\mathbf{H}) \|\boldsymbol{\delta}\|_2^2
\end{aligned}
\end{equation}

Consequently, the quantization loss is strictly bounded by the Hessian's spectral radius. This derivation concludes the proof for the following theorem stated in the main text:

\begin{theorem}[Spectral Error Bound]
\label{thm:main_text_theorem_proof}
The worst-case degradation of the objective function is strictly bounded by the product of the Hessian's spectral radius and the squared norm of the quantization perturbation:
\begin{equation}
\label{eq:final_bound_proof}
\Delta \mathcal{L} \le \frac{1}{2} \lambda_{\max}(\mathbf{H}) \|\boldsymbol{\delta}\|_2^2
\end{equation}
\end{theorem}

\subsection{Proof of Theorem 3.2}
\label{app:proof_theorem_3_2}

In this section, we provide the complete constructive proof of Theorem 3.2. We demonstrate how the HeRo-Q transformation $T \triangleq D_{\alpha}^{-1}R$ achieves a strictly tighter error upper bound by decoupling and synergistically optimizing the `Geometric Factor' and `Noise Factor' within the quantization error bound.

In LLMs, the geometric structure of the loss landscape exhibits extreme anisotropy. Specifically, the spectrum of the Hessian matrix $H$ possesses a tiny number of extremely large eigenvalues, which determine the spectral radius $\lambda_{\max}(H)$ of the matrix. Consequently, the LLM Hessian matrix follows a long-tail distribution. We provide a rigorous mathematical definition as follows:

\textbf{Definition A.1 (Hessian Diagonal Long-tail)} 
Let $H$ be a $d \times d$ Hessian matrix, with its diagonal elements ordered by magnitude as $|H_{(1)}| \ge |H_{(2)}| \ge \dots \ge |H_{(d)}|$. A significant eigenvalue gap exists, such that a few elements dominate the numerical distribution. Specifically, there exists a non-empty set of outlier indices $\mathcal{S}_{out} = \{k \mid |H_{kk}| \gg 1\}$, such that for any $k \in \mathcal{S}_{out}$, $|H_{kk}| > 10^3$ or even higher. The maximum value $|H_{(1)}|$ is far greater than the overall level of the spectral distribution, satisfying the following condition:
\begin{equation}
|H_{(1)}| \gg \frac{1}{d} \sum_{i=1}^{d} |H_{ii}|
\end{equation}

This long-tail characteristic corresponds geometrically to an extremely elongated "hyper-ellipsoid" loss landscape. The short axes (high-curvature directions) are extremely short and steep, while most directions (low-curvature directions) are very flat. 

The core logic of HeRo-Q is to selectively attenuate the curvature of these few long-tail directions through diagonal smoothing $D_{\alpha}$, thereby fundamentally reducing the model's sensitivity to quantization noise.

\subsubsection{Generalized Error Bound in Transformed Space}

To address the norm scaling issues inherent in non-unitary transformations, we first reconstruct the second-order loss approximation within the transformed coordinate system. 

Let $\mathbf{w}$ denote the original weights and $\tilde{\mathbf{w}} = T^{-1}\mathbf{w}$ be the transformed weights. The quantization operation $Q(\cdot)$ is executed in this transformed space, resulting in the quantization noise vector $\tilde{\boldsymbol{\delta}} \triangleq Q(\tilde{\mathbf{w}}) - \tilde{\mathbf{w}}$. 

According to the Taylor expansion, the loss degradation $\Delta \mathcal{L}$ recovered in the original space is governed by:
\begin{equation}
\Delta \mathcal{L} \approx \frac{1}{2} (T\tilde{\boldsymbol{\delta}})^T \mathbf{H} (T\tilde{\boldsymbol{\delta}}) = \frac{1}{2} \tilde{\boldsymbol{\delta}}^T (T^T \mathbf{H} T) \tilde{\boldsymbol{\delta}} = \frac{1}{2} \tilde{\boldsymbol{\delta}}^T \tilde{\mathbf{H}} \tilde{\boldsymbol{\delta}}
\end{equation}
where $\tilde{\mathbf{H}} \triangleq T^T \mathbf{H} T$ is defined as the \textit{equivalent Hessian matrix}. 

Utilizing the spectral norm property ($\mathbf{x}^T \mathbf{A} \mathbf{x} \le \lambda_{\max}(\mathbf{A}) \|\mathbf{x}\|_2^2$), we derive the total error upper bound:
\begin{equation}
\label{eq:total_bound_complete}
\Delta \mathcal{L} \le \frac{1}{2} \underbrace{\lambda_{\max}(\tilde{\mathbf{H}})}_{\text{Geometric Factor}} \cdot \underbrace{\|\tilde{\boldsymbol{\delta}}\|_2^2}_{\text{Noise Factor}}
\end{equation}

The core of the proof lies in demonstrating that HeRo-Q can jointly minimize the product of these two factors. Expanding the definition of the equivalent Hessian $\tilde{\mathbf{H}}$ using $T = D_{\alpha}^{-1}R$ yields:
\begin{equation}
\tilde{\mathbf{H}} = (D_{\alpha}^{-1}R)^T \mathbf{H} (D_{\alpha}^{-1}R) = R^T (D_{\alpha}^{-1} \mathbf{H} D_{\alpha}^{-1}) R = R^T \hat{\mathbf{H}} R
\end{equation}
Since $R$ is an orthogonal matrix ($R^T = R^{-1}$), the relation $\tilde{\mathbf{H}} = R^{-1} \hat{\mathbf{H}} R$ constitutes a \textit{similarity transformation}. Fundamental linear algebra dictates that similarity transformations preserve the eigenvalues of a matrix. Consequently, the spectral properties of $\tilde{\mathbf{H}}$ remain identical to those of the diagonally scaled matrix $\hat{\mathbf{H}}$, regardless of the rotation $R$.

Thus, the equivalent Hessian $\tilde{\mathbf{H}}$ and the preconditioned Hessian $\hat{\mathbf{H}}$ share the same spectral radius:
\begin{equation}
\lambda_{\max}(\tilde{\mathbf{H}}) = \lambda_{\max}(\hat{\mathbf{H}})
\end{equation}
This establishes that the optimization of the Geometric Factor can be simplified to the analysis of $\hat{\mathbf{H}}$, independent of the rotation matrix $R$.

Considering the element at the $k$-th row and $j$-th column, denoted as $\hat{H}_{kj}$, we perform the matrix multiplication $\hat{\mathbf{H}} = \mathbf{D}_{\alpha}^{-1} \mathbf{H} \mathbf{D}_{\alpha}^{-1}$:
\begin{align}
\hat{H}_{kj} &= \sum_{m} \sum_{n} (D_{\alpha}^{-1})_{km} H_{mn} (D_{\alpha}^{-1})_{nj} \label{eq:matrix_mult_def} \\
\intertext{Given that $D_{\alpha}^{-1}$ is diagonal, the sums collapse to single non-zero terms where $m=k$ and $n=j$:} \notag \\
\hat{H}_{kj}&= (D_{\alpha}^{-1})_{kk} H_{kj} (D_{\alpha}^{-1})_{jj} \label{eq:diagonal_collapse} \\
\intertext{Substituting the scaling factor $(D_{\alpha}^{-1})_{ii} = 1 / |H_{ii}|^{\alpha/2}$:} \notag \\
\hat{H}_{kj}&= \frac{1}{|H_{kk}|^{\alpha/2}} \cdot H_{kj} \cdot \frac{1}{|H_{jj}|^{\alpha/2}} = \frac{H_{kj}}{|H_{kk}|^{\alpha/2} |H_{jj}|^{\alpha/2}} \label{eq:final_form_derived}
\end{align}
This explicit form serves as the basis for the subsequent bound derivation. By analyzing the absolute row sums of the equivalent matrix, a strict upper bound for the spectral radius $\lambda_{\max}(\hat{\mathbf{H}})$ can be derived:

\begin{theorem}[Spectral Radius Upper Bound]
\label{thm:spectral_bound}
The spectral radius $\lambda_{\max}(\hat{\mathbf{H}})$ is strictly bounded by the maximum absolute row sum:
\begin{equation}
\lambda_{\max}(\hat{\mathbf{H}}) \le \max_{k} \left( |\hat{H}_{kk}| + \sum_{j \neq k} |\hat{H}_{kj}| \right)
\end{equation}
\end{theorem}

\begin{proof}
Let $(\lambda, \mathbf{x})$ be an eigenpair of $\hat{\mathbf{H}}$, normalized such that the magnitude of the maximum component is 1 (i.e., $|x_k|=1$, implying $|x_j| \le 1$ for all $j$).

Starting from the $k$-th row of the eigenvalue equation $\sum_{j} \hat{H}_{kj} x_j = \lambda x_k$, we isolate the term involving $\lambda$:
\begin{equation}
(\lambda - \hat{H}_{kk}) x_k = \sum_{j \neq k} \hat{H}_{kj} x_j
\end{equation}
Taking the absolute value on both sides and applying the Triangle Inequality:
\begin{equation}
|\lambda - \hat{H}_{kk}| \cdot 1 = \left| \sum_{j \neq k} \hat{H}_{kj} x_j \right| \le \sum_{j \neq k} |\hat{H}_{kj}| \cdot |x_j| \le \sum_{j \neq k} |\hat{H}_{kj}|
\end{equation}
Applying the Reverse Triangle Inequality ($|\lambda| - |\hat{H}_{kk}| \le |\lambda - \hat{H}_{kk}|$) yields:
\begin{equation}
|\lambda| \le |\hat{H}_{kk}| + \sum_{j \neq k} |\hat{H}_{kj}|
\end{equation}
Since this inequality holds for any eigenvalue, it must hold for the spectral radius.
\end{proof}

Substituting $|\hat{H}_{kk}| = |H_{kk}|^{1-\alpha}$ and the off-diagonal forms into the above results in the final theorem:

\begin{theorem}[Spectral Radius Compression Theorem]
\label{thm:spectral_radius_compression}
The spectral radius of the Hessian after HeRo-Q transformation is strictly bounded by the following surrogate objective function $\mathcal{B}(\alpha)$:
\begin{equation}
\lambda_{\max}(\tilde{\mathbf{H}}) \le \max_{k} \left( |H_{kk}|^{1-\alpha} + \sum_{j \neq k} \frac{|H_{kj}|}{|H_{kk}|^{\alpha/2}|H_{jj}|^{\alpha/2}} \right) \triangleq \mathcal{B}(\alpha)
\end{equation}
\end{theorem}

Based on \textbf{Definition A.1}, the spectral radius is determined by a small number of outlier indices $k \in \mathcal{S}_{out}$, whose core characteristic is that their diagonal magnitudes are significantly greater than unity ($|H_{kk}| \gg 1$). For the rows $k \in \mathcal{S}_{out}$ that dominate the spectral radius, we observe the variation of the two components as $\alpha$ increases:

Exponential Reduction of the Diagonal Term: Given the base $|H_{kk}| \gg 1$, the function $g(\alpha) = |H_{kk}|^{1-\alpha}$ is an exponentially decaying function with respect to the base $|H_{kk}|$. As $\alpha$ increases from $0$, this term decreases at an extremely rapid rate.

Smoothing Suppression of Off-diagonal Terms: For the off-diagonal sum $\sum_{j \neq k} \frac{|H_{kj}|}{(|H_{kk}||H_{jj}|)^{\alpha/2}}$, since the product $|H_{kk}||H_{jj}|$ is also significantly greater than $1$ (guaranteed by the long-tail assumption), the denominator grows exponentially as $\alpha$ increases. Consequently, the sum of the entire off-diagonal components strictly decreases.

Since every component of the surrogate function $\mathcal{B}(\alpha)$ behaves as a strictly decreasing function of $\alpha$ under the condition $|H_{kk}| \gg 1$, it follows that $\mathcal{B}(\alpha)$ must be monotonically decreasing within the initial interval $\alpha \in (0, \alpha^*]$. 

Through the analysis above, we have formally proven Theorem 3.2 (Spectral Compression Guarantee) in the main text: under the assumption that the LLM Hessian satisfies the diagonal long-tail distribution, there exists an interval where the loss upper bound after the HeRo-Q transformation is strictly lower than the untreated baseline. 
\begin{equation}
\exists \alpha^* > 0 \text{ s.t. } \mathcal{B}(\alpha^*) < \mathcal{B}(0).
\end{equation}

This theoretical result perfectly supports the first half of the U-shaped curve observed in experiments, demonstrating the decisive role of the smoothing mechanism in suppressing pathological curvature.

\subsection{Formal Analysis of Noise Factor: Outlier Amplification and Suppression}

While the smoothing parameter $\alpha$ successfully compresses the Geometric Factor $\mathcal{B}(\alpha)$ by reducing the Hessian's spectral radius, the quantization noise norm $\|\tilde{\boldsymbol{\delta}}\|_2^2$ in the transformed space is not invariant. We analyze this phenomenon through the lens of \textit{inference equivalence}.

To maintain the output of the linear layer $\mathbf{Y} = \mathbf{X}\mathbf{W}$ after introducing the transformation $\mathbf{T}$, the following identity must hold:
\begin{equation}
    \mathbf{Y} = \mathbf{X}\mathbf{I}\mathbf{W} = (\mathbf{X}\mathbf{T})(\mathbf{T}^{-1}\mathbf{W}) = \tilde{\mathbf{X}}\tilde{\mathbf{W}}
\end{equation}
Consequently, the transformed weights to be quantized are defined as $\tilde{\mathbf{W}} = \mathbf{T}^{-1}\mathbf{W}$. Substituting the definition $\mathbf{T} = \mathbf{D}_{\alpha}^{-1}\mathbf{R}$ and applying the matrix inversion rule $(\mathbf{AB})^{-1} = \mathbf{B}^{-1}\mathbf{A}^{-1}$ along with the orthogonality of the rotation matrix ($\mathbf{R}^{-1} = \mathbf{R}^{\top}$), we derive the transformed weights as:
\begin{equation}
    \tilde{\mathbf{W}} = (\mathbf{D}_{\alpha}^{-1}\mathbf{R})^{-1}\mathbf{W} = \mathbf{R}^{\top}\mathbf{D}_{\alpha}\mathbf{W}
\end{equation}

Given the diagonal smoothing matrix $\mathbf{D}_{\alpha} = \text{diag}(|H_{ii}|^{\alpha/2})$, increasing $\alpha$ effectively attenuates the Hessian eigenvalues in the curvature space. However, due to the inverse relationship, it acts as a direct scaling amplifier in the weight space. This non-uniform scaling leads to a significant expansion of the $L_{\infty}$ norm of the weight vector:
\begin{equation}
    \|\tilde{\mathbf{W}}\|_{\infty} \propto \max_{k} \left( |H_{kk}|^{\alpha/2} |W_{k}| \right)
\end{equation}

Since Large Language Models (LLMs) exhibit extreme outliers ($|H_{kk}| \gg 1$), even a marginal increase in $\alpha$ forces the quantizer to adopt a larger step size $\Delta$ to cover the drastically expanded dynamic range. This results in a quadratic growth of the expected quantization noise variance:
\begin{equation}
    \mathbb{E}[\|\tilde{\boldsymbol{\delta}}\|_2^2] \propto \Delta^2 \propto \|\tilde{\mathbf{W}}\|_{\infty}^2
\end{equation}

Thus, Theorem 3.2 formalizes a fundamental optimization problem: finding the optimal $\alpha^*$ to balance the gain from `spectral radius compression' against the cost of 'dynamic range expansion.' This explains the U-shaped performance curve observed in experiments:

Insufficient compression of the Hessian spectral radius $\lambda_{\max}(\tilde{\mathbf{H}})$ makes the model highly sensitive to even minor quantization noise.
Excessive expansion of the weight dynamic range causes the \textit{Noise Factor} to dominate, degrading performance through increased clipping and rounding errors.

As the Geometric Factor $\mathcal{B}(\alpha)$ is invariant to the rotation $\mathbf{R}$, the role of $\mathbf{R}$ is strictly decoupled to minimize the Noise Factor by redistributing weight energy more uniformly across channels:
\begin{equation}
    \mathbf{R}^* = \arg\min_{\mathbf{R} \in SO(d)} \|\text{Quantize}(\mathbf{R}^{\top}\mathbf{D}_{\alpha}\mathbf{W}) - \mathbf{R}^{\top}\mathbf{D}_{\alpha}\mathbf{W}\|_2^2
\end{equation}

\subsection{Inference overhead}
\label{app:efficiency_analysispp}

As illustrated in Figure \ref{fig:inference}, HeRo-Q adopts a hybrid inference strategy that strictly distinguishes between mergeable transformations (offline fusion) and online transformations to minimize runtime latency.

For the vast majority of linear layers (e.g., the Query/Key/Value projections in Attention blocks and the Up/Gate/Down projections in MLP blocks), the transformation matrix $\mathbf{T} \triangleq \mathbf{D}_{\alpha}^{-1}\mathbf{R}$ can be mathematically fused into the adjacent weight matrices before inference. 
As shown in the purple blocks in Figure \ref{fig:inference}, we apply the following equivalent structural re-parameterization:

The inverse smoothing diagonal $\mathbf{D}_{\alpha}^{-1}$ is absorbed into the preceding normalization layers (LayerNorm). This is a standard element-wise rescaling that incurs no additional operators.

The dense rotation matrix $\mathbf{R}$ is fused into the weight matrices. For a linear layer $\mathbf{Y} = \mathbf{X}\mathbf{W}$, the operation becomes $\mathbf{Y} = (\mathbf{X}\mathbf{R})(\mathbf{R}^{\top}\mathbf{W})$. We pre-compute $\mathbf{W}' = \mathbf{R}^{\top}\mathbf{W}$ offline. Similarly, the preceding layer's output weights are updated to $\mathbf{W}_{prev}' = \mathbf{W}_{prev}\mathbf{R}$ to produce the rotated activation $\mathbf{X}\mathbf{R}$.

Since these operations are completed during the model saving phase, the inference-time computational complexity for these layers remains strictly unchanged at $\mathcal{O}(0)$ additional cost.

In specific scenarios where weight fusion is structurally blocked—such as the Key-Value (KV) Cache quantization or activations preceding Rotary Positional Embeddings (RoPE)—we employ online transformations (depicted as grey blocks in Figure \ref{fig:inference}).
To prevent the $\mathcal{O}(d^2)$ complexity of standard matrix multiplication from becoming a bottleneck, we follow the approach of SpinQuant by adopting \textit{Hadamard Rotation}. 
By utilizing the Fast Walsh-Hadamard Transform (FWHT) kernel, the time complexity of the online rotation is reduced from quadratic to log-linear:
\begin{equation}
    \text{Complexity: } \mathcal{O}(d^2) \rightarrow \mathcal{O}(d \log d)
\end{equation}
where $d$ is the hidden dimension. Given that $d \log d \ll d^2$ for large language models, this overhead is negligible compared to the heavy GEMM (General Matrix Multiply) operations in linear layers. Consequently, HeRo-Q achieves inference throughput comparable to standard FP16 and GPTQ baselines, as empirically verified in Table \ref{tab:efficiency_alpha}(a).
\begin{figure}[t]
    \centering
    \includegraphics[width=0.75\linewidth]{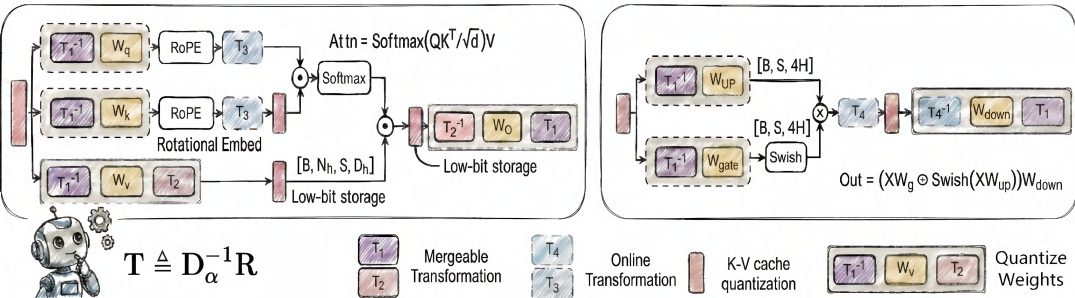}
    \caption{Overview of the HeRo-Q Inference Pipeline.We adopt a hybrid deployment strategy for the transformation $\mathbf{T} \triangleq \mathbf{D}_{\alpha}^{-1}\mathbf{R}$. Purple blocks denote mergeable transformations, where $\mathbf{T}$ or $\mathbf{T}^{-1}$ is fused offline into the adjacent weights ($\mathbf{W}_{q,k,v,o}$ and MLP weights), incurring zero runtime overhead. Grey blocks indicate \textit{online transformations} required where structural fusion is blocked (before RoPE or for KV-cache quantization), which are executed efficiently using optimized kernels.}
    \label{fig:inference}
\end{figure}

\subsection{Calibration Cost Analysis}

While HeRo-Q introduces an iterative optimization phase to learn the rotation matrix $R$, we demonstrate that its calibration overhead remains within a highly practical range for industrial deployment. We evaluated the wall-clock time and peak GPU memory usage during the quantization calibration phase on a single NVIDIA A800 (80GB) GPU. The results, compared with standard baselines, are summarized in Table~\ref{tab:efficiency_comparison}.

\begin{table}[h]
    \centering
    \begin{tabular}{l|l|c|c|l}
    \toprule
    \textbf{Category} & \textbf{Method} & \textbf{Time} & \textbf{Peak VRAM}  \\
    \midrule
    \multirow{2}{*}{Baselines} & GPTQ & $\sim$25 mins & $\sim$16 GB  \\
     & OmniQuant & $\sim$74 mins & $\sim$30 GB  \\
    \midrule
    Ours & HeRo-Q & $\sim$38 mins & $\sim$24 GB  \\
    \bottomrule
    \end{tabular}
    \caption{Comparison of quantization time and memory cost on Llama-3.1-8B (NVIDIA A800).}
    \label{tab:efficiency_comparison}
\end{table}

In terms of computational efficiency, the total time cost to calibrate the Llama-3.1-8B model using HeRo-Q is approximately 38 minutes. Although compared to the gradient-free baseline method GPTQ (~25 minutes), which relies solely on closed-form Hessian inversion, HeRo-Q indeed brings a moderate increase in latency, this additional cost is fully justified. This is because we obtained critical performance gains in the ultra-low-bit regime, such as effectively preventing model collapse under W3A16 settings. Furthermore, when compared with other optimization-based quantization frameworks like OmniQuant—the latter requiring extensive gradient updates for learnable clipping and weights (~74 minutes)—HeRo-Q achieved a significant 1.9x speedup. This high efficiency stems from our streamlined optimization objective: we exclusively optimize the orthogonal rotation matrix $R$ for a limited budget of 200 steps per layer, thereby avoiding the heavy computational burden brought by full-parameter fine-tuning while ensuring fast convergence. 

In terms of memory footprint, HeRo-Q exhibits remarkable resource efficiency, requiring only approximately 24 GB of peak VRAM during the calibration phase. This figure is significantly lower than the memory demand of OmniQuant (~30 GB), and crucially, it falls completely within the memory capacity range of high-end consumer-grade GPUs such as the NVIDIA RTX 3090 or 4090. This characteristic significantly democratizes access to high-performance quantization, enabling the algorithm to be deployed on accessible hardware without strict reliance on data center-grade infrastructure.

\subsection{Limitations}

Although HeRo-Q provides a general Hessian-conditioned framework for improving low-bit Post-training quantization, it has several limitations.

First, HeRo-Q relies on an activation-induced Hessian proxy, i.e., $H_{\mathrm{proxy}} = X^\top X$, estimated from a calibration set. This proxy is efficient and widely used in layer-wise PTQ reconstruction, but it may not perfectly match the true Hessian of the end-task loss. Therefore, the effectiveness of the proposed transformation depends on the representativeness of the calibration data. If the calibration distribution differs substantially from the deployment or evaluation distribution, the estimated curvature-sensitive directions may be biased, and the benefit of Hessian-conditioned smoothing and rotation may be reduced.

Second, our theoretical analysis is based on a local second-order approximation of the loss landscape. This approximation is most reliable around a well-converged model and for sufficiently small perturbations. Under extremely aggressive quantization, such as very low activation precision or extremely small weight bit-widths, clipping error, rounding saturation, and nonlinear effects may become dominant. In such cases, the quadratic surrogate may no longer fully characterize the actual task-level degradation.

Third, HeRo-Q introduces additional offline calibration cost compared with simpler PTQ methods. The method searches the smoothing coefficient $\alpha$ and optimizes or selects an orthogonal rotation for each layer. Although these operations are performed only once before deployment and do not require end-to-end model fine-tuning, they increase the calibration time and memory usage relative to purely reconstruction-based or range-based quantization baselines. This overhead should be considered when quantizing very large models or when frequent re-quantization is required.

Finally, while our experiments cover dense LLMs, vision models, multimodal models, and MoE models under several low-bit settings, they do not exhaust all possible architectures, calibration domains, decoding configurations, or hardware platforms. The reported results are mainly based on standard public benchmarks and a fixed set of calibration protocols. Future work should further study HeRo-Q under broader domain shifts, longer-context inference, different calibration budgets, more hardware backends, and repeated runs with different random seeds or calibration subsets.


\newpage
\section*{NeurIPS Paper Checklist}

\begin{enumerate}

\item {\bf Claims}
    \item[] Question: Do the main claims made in the abstract and introduction accurately reflect the paper's contributions and scope?
    \item[] Answer: \answerYes{}
    \item[] Justification: The abstract and introduction accurately summarize the main contribution of HeRo-Q: a Hessian-conditioned post-training quantization framework that combines diagonal smoothing and orthogonal rotation to improve low-bit quantization robustness. The stated claims are supported by the proposed algorithm, theoretical motivation, downstream evaluations, spectral analysis, efficiency measurements, and ablation studies.

\item {\bf Limitations}
    \item[] Question: Does the paper discuss the limitations of the work performed by the authors?
    \item[] Answer: \answerYes{}
    \item[] Justification: The paper includes a dedicated limitations discussion. It discusses the dependence on the activation-induced Hessian proxy, calibration data representativeness, the local nature of the second-order approximation, the offline calibration and search cost, model-structure-dependent fusion, possible online transformation overhead in certain components, and the scope of the evaluated architectures, calibration domains, and hardware settings.

\item {\bf Theory Assumptions and Proofs}
    \item[] Question: For each theoretical result, does the paper provide the full set of assumptions and a complete (and correct) proof?
    \item[] Answer: \answerYes{}
    \item[] Justification: The paper states the main theoretical results in the methodology section and provides the corresponding derivations and proofs in the appendix. The analysis specifies the local quadratic approximation, the effective Hessian in the quantization coordinate system, the spectral surrogate bound, the Gershgorin-type upper bound, and the assumptions under which Hessian-guided smoothing can reduce the curvature-amplification bound.

\item {\bf Experimental Result Reproducibility}
    \item[] Question: Does the paper fully disclose all the information needed to reproduce the main experimental results of the paper to the extent that it affects the main claims and/or conclusions of the paper, regardless of whether the code and data are provided or not?
    \item[] Answer: \answerYes{}
    \item[] Justification: The paper specifies the evaluated model families, datasets, quantization regimes, baselines, calibration setup, evaluation protocols, and main hyperparameters. The appendix further provides the anonymized code repository and clarifies the layer-wise smoothing search grid, rotation optimization procedure, calibration dataset, number of calibration samples, and sequence length.

\item {\bf Open access to data and code}
    \item[] Question: Does the paper provide open access to the data and code, with sufficient instructions to faithfully reproduce the main experimental results, as described in supplemental material?
    \item[] Answer: \answerYes{}
    \item[] Justification: The appendix provides an anonymized code repository link. The paper uses publicly available models and datasets, and the appendix describes the main workflow and hyperparameters required to reproduce the reported results. The repository contains code and instructions for calibration, Hessian-conditioned transformation learning, quantization, and evaluation.

\item {\bf Experimental Setting/Details}
    \item[] Question: Does the paper specify all the training and test details, including data splits, hyperparameters, how they were chosen, optimizer choices, and other details necessary to understand the results?
    \item[] Answer: \answerYes{}
    \item[] Justification: The experimental section describes the evaluated models, datasets, downstream tasks, quantization settings, baseline methods, and evaluation protocols. The appendix provides the HeRo-Q hyperparameters, including the smoothing search grid, Cayley SGD optimizer, learning rate, optimization steps, momentum, batch size, calibration dataset, number of calibration samples, and sequence length.

\item {\bf Experiment Statistical Significance}
    \item[] Question: Does the paper report error bars suitably and correctly defined or other appropriate information about the statistical significance of the experiments?
    \item[] Answer: \answerNo{}
    \item[] Justification: The paper reports single-run evaluation results and does not provide error bars, confidence intervals, or formal statistical significance tests. This is a limitation of the current empirical evaluation. Future work should report variation across random seeds or different calibration subsets, especially for the main low-bit quantization results.

\item {\bf Experiments Compute Resources}
    \item[] Question: For each experiment, does the paper provide sufficient information on the computer resources, including type of compute workers, memory, and execution time, needed to reproduce the experiments?
    \item[] Answer: \answerYes{}
    \item[] Justification: The paper reports the main compute environment and resource usage, including the NVIDIA A800 GPU used for throughput measurements, calibration time, and peak VRAM usage. These details provide a reasonable estimate of the computational resources required to reproduce the main experiments.

\item {\bf Code Of Ethics}
    \item[] Question: Does the research conducted in the paper conform, in every respect, with the NeurIPS Code of Ethics?
    \item[] Answer: \answerYes{}
    \item[] Justification: This is a model quantization method paper. It does not involve human subjects, private data, restricted datasets, or the release of a new high-risk generative model. The submission preserves anonymity and discusses responsible-use implications in the broader impact statement.

\item {\bf Broader Impacts}
    \item[] Question: Does the paper discuss both potential positive societal impacts and negative societal impacts of the work performed?
    \item[] Answer: \answerYes{}
    \item[] Justification: The paper discusses positive impacts such as reducing the hardware, memory, and energy costs of deploying large models. It also discusses potential negative impacts, including the possibility that more efficient compressed models may lower the barrier for misuse, unsafe deployment, or broader dissemination of harmful model capabilities.

\item {\bf Safeguards}
    \item[] Question: Does the paper describe safeguards that have been put in place for responsible release of data or models that have a high risk for misuse, such as pretrained language models, image generators, or scraped datasets?
    \item[] Answer: \answerNA{}
    \item[] Justification: The paper proposes a quantization method and does not release a new pretrained language model, image generator, scraped dataset, or other high-risk artifact. The released artifact is an anonymized implementation of the proposed quantization framework, so additional model-release safeguards are not applicable.

\item {\bf Licenses for existing assets}
    \item[] Question: Are the creators or original owners of assets, such as code, data, and models, used in the paper properly credited, and are the license and terms of use explicitly mentioned and properly respected?
    \item[] Answer: \answerYes{}
    \item[] Justification: The paper cites the external models, datasets, benchmarks, and baseline methods used in the experiments. The asset and license statement summarizes the public assets used in this work and indicates that the corresponding sources, versions, licenses, and terms of use are documented in the appendix or released repository.

\item {\bf New Assets}
    \item[] Question: Are new assets introduced in the paper well documented, and is the documentation provided alongside the assets?
    \item[] Answer: \answerYes{}
    \item[] Justification: The paper releases an anonymized code artifact for HeRo-Q. The repository includes documentation, environment requirements, example commands, scripts for reproducing the main experiments, and a license file. The paper does not introduce a new dataset or a newly pretrained model.

\item {\bf Crowdsourcing and Research with Human Subjects}
    \item[] Question: For crowdsourcing experiments and research with human subjects, does the paper include the full text of instructions given to participants and screenshots, if applicable, as well as details about compensation?
    \item[] Answer: \answerNA{}
    \item[] Justification: The work does not involve crowdsourcing, user studies, annotation by human participants, or research with human subjects.

\item {\bf Institutional Review Board (IRB) Approvals or Equivalent for Research with Human Subjects}
    \item[] Question: Does the paper describe potential risks incurred by study participants, whether such risks were disclosed to the subjects, and whether Institutional Review Board approvals or an equivalent approval/review were obtained?
    \item[] Answer: \answerNA{}
    \item[] Justification: The work does not involve crowdsourcing or research with human subjects, so IRB approval or equivalent review is not applicable.

\end{enumerate}

\end{document}